\documentclass[letterpaper, 10 pt, conference]{ieeeconf}
\IEEEoverridecommandlockouts                              

\usepackage{amsmath,amsfonts,amssymb,color,amsthm}
\usepackage{nccmath}
\usepackage{algorithm}
\usepackage{algpseudocode}
\usepackage{array}

\usepackage{graphicx}
\usepackage{cite}
\usepackage{footmisc}
\usepackage{mathtools}
\usepackage{bbm}
\usepackage{bm}
\usepackage{comment}
\usepackage{dsfont}

\usepackage{mathtools}
\usepackage{medmath}
\usepackage{tikz}
\usetikzlibrary{automata, positioning, arrows.meta}
\usepackage{relsize}
\usetikzlibrary{shapes,arrows}
\usepackage{cite}

\usepackage[colorlinks]{hyperref}
\definecolor{mygreen}{rgb}{0.0, 0.5, 0.0}
\definecolor{winered}{rgb}{0.8,0,0}
\definecolor{myblue}{rgb}{0,0,0.8}

\usepackage{scalerel}
\usepackage{xcolor}

\DeclarePairedDelimiterX{\norm}[1]{\lVert}{\rVert}{#1}

\newcommand{\mc}{\mathcal}

\AtBeginDocument{%
  \hypersetup{
    citecolor=myblue,
    linkcolor=winered}}

\newcommand{\smallsquare}{\scalebox{0.7}{$\blacksquare$}}

\newtheorem{problem}{Problem}

\newtheorem{theorem}{Theorem}

\newtheorem{lemma}{Lemma}
\newtheorem{remark}{Remark}

\newtheorem{assumption}{Assumption}

\title{\LARGE \bf Robust Asynchronous Q-Learning under Reward and State Corruption via Batching}
\author{Sreejeet Maity and Aritra Mitra 
\thanks{The authors are with the Department of Electrical and Computer Engineering, North Carolina State University. Email: {\tt \{smaity2, amitra2\}@ncsu.edu}.}}
\begin{document}
\maketitle
\begin{abstract} Motivated by reinforcement learning in harsh environments, we consider the problem of learning an optimal policy subject to adversarially corrupted feedback. Specifically, at each time-step, an adversary can perturb both the reward and state observations of the learner following the Huber contamination model. To defend against such data corruption, 
we propose {{\texttt{BR-Async-Q}}}: a novel, epoch-based, robust \(Q\)-learning algorithm built upon two key ideas: (i) partitioning the online data stream into batches to reduce variance, and (ii) constructing robust estimates of the Bellman optimality operator using such batched data. We prove a high-probability $\ell_\infty$ error bound for {{\texttt{BR-Async-Q}}} that matches that for vanilla \(Q\)-learning, up to a small additive term that scales with the fraction of corrupted samples. To our knowledge, this provides the first robustness guarantee for asynchronous \(Q\)-learning subject to both reward and state corruption. Furthermore, when only rewards are corrupted, the dependence of our algorithm's bound on the corruption fraction is minimax optimal. 
\end{abstract}
\section{Introduction}
\label{sec:Intro}
Reinforcement learning (RL) has emerged as a powerful framework for sequential decision-making, with successful applications spanning healthcare, recommendation systems, autonomous driving, and robotics. At its core, an agent repeatedly interacts with an \emph{unknown} environment, observes rewards and successor states, and uses this feedback to improve its policy and maximize long-term return. In practice, however, this feedback is often \emph{imperfect}: rewards may be noisy, biased, or adversarially manipulated, and observed next states may be distorted by sensing errors, logging failures, or targeted perturbations. Such imperfections are particularly harmful in sequential learning algorithms where corrupted feedback-induced biases can compound over time, generating \emph{poor decision-making policies.}

\textbf{The Setting.} Motivated by the gap between idealized and noisy feedback, we study discounted infinite-horizon RL under \emph{adversarially corrupted data}. We consider learning from an online trajectory generated by a behavior policy, as in \(Q\)-learning~\cite{watkins1992q}, but allow an adversary to corrupt both the \emph{realized reward} and the \emph{observed next state} at each step, thereby distorting the learner’s update signal. We model corruption using the classical Huber model~\cite{huber,huber2}: at each time step, the observed next state and reward can be \emph{arbitrarily} perturbed with a small corruption probability $\varepsilon \in [0, 1/2).$ Moreover, we even allow the clean reward distributions to be heavy-tailed, requiring them to admit no more than a second moment. This makes it particularly challenging for the learner to distinguish between deliberately corrupted reward signals and  clean rewards generated from the tails of reward distributions. Subject to our observation model, which captures severe noise, sensing failures, and deliberate poisoning, we ask: \emph{Is it still possible to reliably learn an optimal policy?} As we discuss below, there are several gaps in our understanding of this question.

\textbf{Related Work.}
Although robustness in sequential decision-making has been extensively studied for simpler settings such as stochastic bandits~\cite{lykouris, kapoor}, much less is known for the online, infinite-horizon, discounted RL setting we consider in this paper. A very recent line of work~\cite{maity2025corruption,maity2025robust,maity2024robust} has initiated the study of robust RL under \emph{reward-only} corruption, assuming perfect state-feedback at all times. However, even for this weaker corruption model, there is a fundamental gap between the upper and lower bounds derived in~\cite{maity2025corruption, maity2025robust}; we explain this gap in Section~\ref{sec:Problem Formulation}. Furthermore, prior work has left it unclear how to simultaneously control the effects of reward and state corruptions. In this paper, we propose a novel approach that resolves both the aforementioned issues. Our contributions are as follows. 

$\bullet$ \textbf{Algorithmic Contribution.} We develop \textcolor{winered}{\texttt{BR-Async-Q}} in Section~\ref{sec:braceQ}, an epoch-based, robust variant of \(Q\)-learning, that can handle asynchronous online sampling and joint corruption of states and rewards. Our approach partitions the online data stream into batches (epochs) and relies on computing a robust, low-variance estimate of the Bellman optimality operator within each epoch. Unlike standard \(Q\)-learning~\cite{watkins1992q} and prior robust variants thereof~\cite{maity2024robust, maity2025corruption} which update the \(Q\)-table at all time steps, \textcolor{winered}{\texttt{BR-Async-Q}} updates the \(Q\)-table infrequently, only once per epoch, but using \emph{variance-reduced} update directions. It is precisely variance-reduction via batching that enables us to close the gap between the upper and lower bounds in prior work~\cite{maity2025corruption, maity2025robust}. 

$\bullet$ \textbf{Theoretical Contribution.} In Theorem~\ref{thm:main theorem 1} of Section~\ref{sec:main_result}, we prove a high-probability finite-time \(\ell_\infty\)-error bound for \textcolor{winered}{\texttt{BR-Async-Q}}, showing that one can still converge ``close" to the optimal state-action value function. Given $T$ samples, our bound features a statistical error term on the order of $\tilde{\mc{O}}(1/\sqrt{T})$, which, in fact, sharpens even the standard \(Q\)-learning bound in~\cite{Qu} in terms of its dependence on the visitation probability of the least-visited state-action pair. Our bound also exhibits a small unavoidable additive bias term that scales with the corruption probability. When only rewards are corrupted, our bias term is minimax optimal, \emph{exactly} matching the fundamental lower bound in~\cite{maity2025corruption}. Thus, our work not only improves existing finite-time rates, but also provides the first provable guarantee of robustness for \(Q\)-learning under the joint state and reward corruption model. 

\section{Background and Problem Formulation} 
\label{sec:Problem Formulation}
We begin with the standard RL background and then describe the problem of interest. Consider a \(\gamma\)-discounted infinite-horizon Markov Decision Process (MDP) \(\mathcal{M}=(\mathcal{S},\mathcal{A},\mathcal{P},R,\gamma)\), where \(\mathcal{S}\) and \(\mathcal{A}\) are finite state and action spaces, \(\mathcal{P}\) is the probability transition kernel, \(R\) is the mean reward function, and \(\gamma\in(0,1)\) is the discount factor. When in state-action pair \((s,a)\), the learner observes a next state \(s'\sim \mc P(\cdot\mid s,a)\) and a stochastic reward \(r(s,a)\sim \mc R(s,a)\) drawn from the reward distribution $\mc R(s,a)$, such that \(\mathbb{E}[r(s,a)] = R(s,a)\) and \(\mathbb{E}[(r(s,a)-R(s,a))^2] = \sigma^2(s,a)\). We assume that the reward means and variances are bounded, i.e., $\exists \bar R,\bar\sigma\ge 1$ such that \(|R(s,a)|\le \bar R\) and \(\sigma^2(s,a)\le \bar\sigma^2\) for all \((s,a)\in\mc S\times\mc A\). We define \(\tilde{\sigma}:=\max\{\bar R,\bar\sigma\}\). 

A policy \(\mu:\mc S\to\Delta(\mc A)\) assigns to each state a distribution over actions. The quality of a policy $\mu$ is captured by an expected discounted infinite-horizon cumulative reward known as the value function $V^{\mu}$, defined as
\(\medmath{V^\mu(s) = \mathbb{E}\left[ \sum_{t=0}^\infty \gamma^t R(s_t, a_t) \,\bigg|\, s_0 = s, \mu \right],}\) where $s_t$ and $a_t$ are the state and action at time $t$, respectively, under the action of the policy $\mu$ on the MDP $\mc{M}$. The goal of the learner is to find a policy $\mu$ that maximizes the value function $V^\mu$ for all states, \emph{without knowledge} of the transition kernels $\mc{P}$ and reward functions $R$ of the underlying MDP. To explain how this is done, we will need to introduce the notion of a state-action value function $Q^\mu $ for a policy $\mu$, defined as
\begin{equation}\medmath{
    Q^\mu(s,a) = \mathbb{E}\left[ \sum_{t=0}^\infty \gamma^t R(s_t,a_t) \,\bigg|\, (s_0, a_0) = (s, a), \mu \right].}
\end{equation} 
The celebrated $Q$-learning algorithm~\cite{watkins1992q} uses data collected by a suitable behavior/sampling policy $\mu$ to iteratively maintain an estimate of the optimal state-action value function, denoted by $Q^*$. It turns out that $Q^*$ is the fixed point of a contractive operator $\mc{T}$ known as the Bellman optimality operator~\cite{suttonRL}. Using this contraction property, classical asymptotic results~\cite{tsitsiklis94, jaakkola} established that the sequence of iterates generated by \(Q\)-learning converges to $Q^*$ almost surely under suitable assumptions on $\mu$. More recently, finite-time rates have been established~\cite{Waiwright, Qu, li2024q}, revealing that when run for $T$ iterations, the final iterate of $Q$-learning converges to $Q^*$ at a rate of $\tilde{\mc{O}}(1/\sqrt{T})$, with high probability. 


\textbf{Adversarially Corrupted Observation Model.} As is standard in \(Q\)-learning~\cite{chen2022Auto, li2024q}, we assume that data are generated under a stochastic behavior policy \(\mu\) satisfying \(\mu(a\mid s)>0\) for all \((s,a)\). In the nominal model, the learner observes \(a_t\sim \mu(\cdot\mid s_t)\), \(s_{t+1}\sim \mc P(\cdot\mid s_t,a_t)\), and \(r_t(s_t,a_t)\sim \mc R(s_t,a_t)\). Here, we assume that the noise process $\{R(s_t,a_t) - r_t(s_t,a_t)\}$ is independent over time and of all other sources of randomness. Our formulation departs from standard \(Q\)-learning in two key aspects. First, unlike prior work which either requires the rewards to be deterministic or ``light-tailed" (e.g., sub-Gaussian), we impose nothing beyond the reward distributions $\mc{R}(s,a)$ admitting a finite second moment, i.e., we allow for \emph{heavy-tailed} rewards. Second, the learner's feedback can be adversarially corrupted according to the classical Huber contamination model~\cite{huber}. 

To explain this corruption model, let \(\{Y_{t,1}\}_{t\ge 0}\) and \(\{Y_{t,2}\}_{t\ge 0}\) be two i.i.d.\ Bernoulli (\texttt{Bern}) sequences with parameters \(\varepsilon_{\mc R},\varepsilon_{\mc Y}\in[0,1/2)\), respectively, independent of the past and of all other sources of randomness. Note, however, that the processes \(\{Y_{t,1}\}_{t\ge 0}\) and \(\{Y_{t,2}\}_{t\ge 0}\) might be mutually correlated. At each time step \(t\), the learner observes \((s_t,a_t)\), and then receives an \emph{observed} reward \(\widetilde r_t\) and an \emph{observed} next state \(\widetilde s_{t+1}\) that are generated as follows: 
\begin{equation}\label{eqn:obs_corruption_model}
\begin{aligned}
\widetilde r_t(s_t,a_t) &= (1-Y_{t,1})\,r_t(s_t,a_t) + Y_{t,1}\,z_t,\\
\widetilde s_{t+1} &= (1-Y_{t,2})\,s_{t+1} + Y_{t,2}\,u_t,
\end{aligned}
\end{equation}
where \(Y_{t,1}\overset{\mathrm{i.i.d.}}{\sim}\texttt{Bern}(\textcolor{winered}{\varepsilon_{\mc R}})\), \(Y_{t,2}\overset{\mathrm{i.i.d.}}{\sim}\texttt{Bern}(\textcolor{winered}{\varepsilon_{\mc Y}})\), \(a_t\sim\mu(\cdot\mid s_t)\), \(s_{t+1}\sim\mc P(\cdot\mid s_t,a_t)\), and \(r_t(s_t,a_t)\sim\mc R(s_t,a_t)\). Here, \(z_t\) is an \emph{arbitrary} (possibly unbounded) adversarial reward value, and \(u_t\in\mc S\) is an \emph{arbitrary} adversarially chosen next state. Both \(z_t\) and \(u_t\) may depend on the entire history up to time \(t\). In simple words, at each time-step $t$, with probability $1-\varepsilon_{\mc R}$ (resp., $1-\varepsilon_{\mc Y}$), the learner observes a noisy reward (resp., next state) from the true reward distribution $\mc{R}(s_t, a_t)$ (resp., true state distribution $\mc{P}(\cdot| s_t, a_t)$), and with probability $\varepsilon_{\mc R}$ (resp., $\varepsilon_{\mc Y}$), an arbitrary reward (resp., next state). Our problem of interest can now be summarized as follows. 

\begin{problem}
Given \(T\) samples \(\{(s_t,a_t,\tilde{s}_{t+1},\tilde{r}_t(s_t,a_t))\}_{t=0}^{T-1}\) generated under the adversarial model in~\eqref{eqn:obs_corruption_model}, and a failure probability \(\delta \in (0,1) \), our goal is to construct a robust estimate \(\hat{Q}\) of the optimal state-action value function \(Q^*\), and derive a high-probability bound on the \(\ell_\infty\)-error \(\|\hat{Q}-Q^*\|_\infty\) that holds with probability at least \(1-\delta\).
\end{problem}

\noindent \textbf{Challenges and Research Gaps.} The standard \(Q\)-learning update~\cite{watkins1992q} uses a ``bootstrapping" mechanism where the current estimate of the \(Q\)-table is used to approximate the ``reward-to-go". The \textbf{joint} reward and state observation corruption model in~\eqref{eqn:obs_corruption_model} not only introduces a bias in the current reward but also in such a bootstrapped estimate. Unless accounted for carefully, this bias might accumulate over time, leading to vacuous bounds. Indeed, even when only rewards are perturbed, such a phenomenon is reported in~\cite{maity2024robust}. As far as we are aware, the more challenging joint corruption model we study here has not been explored at all in prior \(Q\)-learning literature. Furthermore, restricted to the simpler reward corruption model (i.e., when $\varepsilon_{\mc{Y}} =0$) under asynchronous sampling, the upper-bounds derived in~\cite{maity2025corruption} exhibit a corruption-induced $\mc{O}(\sqrt{\varepsilon_{\mc{R}}})$ term that scales inversely with the visitation probability of the least-visited state-action pair. This seems to suggest that if the data-generation behavior policy $\mu$ causes certain state-action pairs to be visited infrequently, then the effect of reward corruption can become pronounced. Interestingly, however,~\cite{maity2025corruption} establishes a fundamental information-theoretic lower bound on the order~$\Omega(\sqrt{\varepsilon_{\mc{R}}})$ that does not exhibit any dependence at all on visitation probabilities. Thus, prior work has left it unclear how to close the aforementioned gap between upper and lower bounds. In this paper, we develop a novel robust \(Q\)-learning algorithm that not only closes this gap, but also handles the joint corruption model. 

To clearly explain the main technical ideas behind our approach, we focus on the tabular setting in this paper. Even for this setting, a non-trivial technical analysis is needed to handle the interplay between asynchronous sampling, heavy-tailed reward distributions, and joint corruptions of states and rewards. To proceed, we make the following assumption that is standard in prior RL work~\cite{tsitsiklisroy,Qu,bhandari_finite, srikant, mitra2024simple}.

\begin{assumption}\label{ass:ergodic}
The Markov chain induced by the behavior policy $\mu$ is aperiodic and irreducible. 
    \end{assumption}
Let \(\pi\) denote the stationary distribution of the Markov chain induced by \(\mu\). Assumption~\ref{ass:ergodic} ensures that \(\pi(s)>0\) for all \(s\in\mc S\). Thus, the stationary visitation probability of any state-action pair \((s,a)\) is \(\lambda(s,a):=\pi(s)\mu(a|s)>0\). We also define the minimum visitation probability by \(\lambda_{\texttt{min}}:=\min_{(s,a)\in\mc S\times\mc A}\lambda(s,a)\). As is fairly common~\cite{korda, narayanan, dalal, bhandari_finite}, we assume that at each time step \(t\), the state \(s_t\) is sampled \emph{independently} from \(\pi\). Throughout the paper, we will refer to this as the ``i.i.d. data generation model". This assumption is made solely for ease of presentation, since, as shown in~\cite{maity2025corruption,maity2025robust}, the analysis extends naturally to the single-trajectory Markovian setting via blocking arguments~\cite{dorfman}.

\section{Robust Estimation under Huber Corruption}
\label{sec:TrimmedMean}
Before presenting our proposed algorithm in Section~\ref{sec:braceQ}, we first describe a key building block, namely Algorithm~\ref{alg:trimmed_huber}, which enables robust mean estimation under Huber contamination. To that end, consider a data set \(\mc{D}\) consisting of \(M\) i.i.d.\ samples of a scalar random variable \(X\) with mean \(\mu_X\) and variance \(\sigma_X^2\). We consider the \emph{contamination model} in~\eqref{eqn:obs_corruption_model}, under which each sample in $\mc{D}$ is independently corrupted with probability \(\varepsilon\), and remains clean otherwise. Let \(\mc{D}'\) denote the resulting corrupted sample set. Note that \(\mc{D}'\) contains $\varepsilon M$ corrupted samples on average; the realized number of corruptions is, however, a random variable.

To estimate \(\mu_X\) robustly from \(\mc{D}'\), we use a slightly modified version of the trimmed mean estimator from~\cite{lugosi}, where the modification accounts for a random number of corrupted samples. We partition \(\mc{D}'\) into two equal halves, \(\mc{D}_1\) and \(\mc{D}_2\). The first half is used to construct lower and upper cutoff levels \((\texttt{a},\texttt{b})\), and the second half is used to compute a clipped empirical mean over this interval; see lines 1-3 of Algorithm~\ref{alg:trimmed_huber}. 
The resulting robust estimator is denoted by
\begin{equation}\label{eqn:huber_trim}
\hat{\mu}_X^{\textcolor{winered}{\textup{H}}}
=
{\texttt{TRIM}}_{{\textup{H}}}[\mc{D}',\varepsilon,\delta]. 
\end{equation}
When the clean samples are known to be bounded in the interval \((\underline{\Gamma},\overline{\Gamma})\), one can obtain a tighter dependence on the corruption fraction $\varepsilon$ by setting the cutoff levels as \((\texttt{a},\texttt{b})=(\underline{\Gamma},\overline{\Gamma})\). In this case, the resulting estimator is given by 
\begin{equation}
\label{eq:refined_trim}
   \hat{\mu}_X^{\textcolor{winered}{\textup{B}}}= {\textup{\texttt{TRIM}}}_{{\textup{B}}}[\mc{D}', (\underline{\Gamma},\overline{\Gamma})] =
   \frac{1}{M}\sum_{X_i\in\mc{D}'}\phi_{\underline{\Gamma},\overline{\Gamma}}(X_i),
\end{equation}
where for $a < b$, we define the function \(\phi_{\texttt{a},\texttt{b}}(x):={\min}\{{\max}\{x,\texttt{a}\},\,\texttt{b}\}, \forall x \in \mathbb{R}\), and $\mc{D}'=\{X_i\}_{i \in [M]}.$ Before proceeding to analyze the estimators in~\eqref{eqn:huber_trim} and~\eqref{eq:refined_trim}, we note that for the quantiles in Line 1 of Algorithm~\ref{alg:trimmed_huber} to be well defined, the object $\zeta$ computed in Algorithm~\ref{alg:trimmed_huber} must be a fraction, implicitly imposing that $\bar{\varepsilon} \leq 1/8.$\footnote{The object $\zeta$ is defined to account for the concentration of (i) empirical quantiles around true quantiles, and (ii) the number of corrupted samples around its mean value.} We will work under this implicit assumption on $\bar{\varepsilon}$ throughout the paper. It is worth emphasizing that such an assumption is also implicitly made in the work of Lugosi~\cite{lugosi}, and papers that have leveraged similar algorithms in other contexts~\cite{yin}.

\begin{algorithm}[h]
\caption{Robust Mean under Huber Corruption (\({\texttt{TRIM}}_{{\textup{H}}}\))}
\label{alg:trimmed_huber}
\begin{algorithmic}[1]
\Require Contaminated dataset  \(\mc{D}'=\{X_1,X_2,\dots,X_M\}=\mc{D}_1\oplus\mc{D}_2\), where \(|\mc{D}_1|=|\mc{D}_2|=M/2\); contamination probability \(\varepsilon\); confidence level \(\delta\). Define
\(\bar{\varepsilon} := 3\varepsilon/2+16\log(8/\delta)/M,
\zeta := 8\bar{\varepsilon} + 24\log(8/\delta)/M.
\)
\State Let $X^*_1 \leq X^*_2 \leq \cdots \leq  X^*_{M/2}$ represent a non-decreasing arrangement of $\mc{D}_1$. Compute \texttt{quantiles}: \(\texttt{a} = X^*_{\zeta M/2}, \quad \texttt{b} = X^*_{(1 - \zeta) M/2}\), where \(\texttt{a} < \texttt{b}\).
\State Define the function \(\phi_{\texttt{a},\texttt{b}}(x):={\min}\{{\max}\{x,\texttt{a}\},\,\texttt{b}\}.\)
\State Output
\(\hat{\mu}_X^{\textcolor{winered}{\textup{H}}}=\frac{2}{M}\sum_{X_i\in\mc{D}_2}\phi_{\texttt{a},\texttt{b}}(X_i).\)
\end{algorithmic}
\end{algorithm}

The following result will play a key role in our subsequent algorithmic design and analysis. 
\begin{theorem}\textbf{(Robust Mean under Huber contamination)}
\label{thm:trim_huber_both}
Let $\delta\in(0,1)$ satisfy $\delta \ge 8e^{-M/2}$.
There exists a universal constant $\mc C\ge 1$ such that,
with probability at least $1-\delta$, the estimate $\hat{\mu}_X^{\textcolor{winered}{\textup{H}}}$ in~\eqref{eqn:huber_trim} satisfies the following bound:
\begin{equation}\label{eq:trim_var_bound}
|\hat\mu_X^{\textcolor{winered}{\textup{H}}}-\mu_X|
\;\le\;
\mc C\,\sigma_X\Bigl(\sqrt{\varepsilon}+\sqrt{\tfrac{\log(8/\delta)}{M}}\Bigr).
\end{equation}
Moreover, if \(X \in[\underline{\Gamma},\overline{\Gamma}]\), then the estimate $\hat{\mu}_X^{\textcolor{winered}{\textup{B}}}$ in~\eqref{eq:refined_trim} satisfies the following bound with probability at least \(1-\delta\),
\begin{equation}\label{eq:trim_bounded_refined}
|\hat{\mu}_X^{\textcolor{winered}{\textup{B}}}-\mu_X|
\;\le\;
\bar{c} \,\left( \overline{\Gamma}-\underline{\Gamma}\right)\Bigl(\varepsilon+\sqrt{\tfrac{\log(8/\delta)}{M}}\Bigr),
\end{equation}
where \(\bar{c} \ge 1\) is a universal constant. 
\end{theorem}

The proof of the claim in~\eqref{eq:trim_var_bound} is a minor modification of that in~\cite{lugosi}, and can be found in~\cite[Appendix~D]{maity2025corruption}. Establishing the claim in~\eqref{eq:trim_bounded_refined} requires a somewhat different argument that we defer to Appendix~\ref{sec:app}. 

\begin{remark}
It is instructive to compare the bound in~\eqref{eq:trim_var_bound} with that in~\eqref{eq:trim_bounded_refined}. While the former exhibits a dependence of $O(\sqrt{\varepsilon})$, the latter has a sharper dependence of $O({\varepsilon})$. This improvement is made possible by using the known bounds on the random variable to inform the trimming process. We will exploit this idea in our algorithm design in the next section. 
\end{remark}
\section{The Proposed Algorithm}
\label{sec:braceQ}
The standard model-free \(Q\)-learning algorithm~\cite{watkins1992q} and its robust counterparts in~\cite{maity2024robust, maity2025corruption} have the following feature in common: they update the \(Q\)-table at \emph{every} time-step as soon as a new data sample arrives. Since an update is made using just a \emph{single} sample, the update direction has a high variance, making it more vulnerable to corruption. Thus, as alluded to in Sections~\ref{sec:Intro} and~\ref{sec:Problem Formulation}, the resulting guarantees in~\cite{maity2025corruption} are loose, and fail to match the fundamental lower bound. To mitigate this challenge, we depart fundamentally from prior approaches and develop a novel, epoch-based robust variant of asynchronous \(Q\)-learning that uses \emph{online batching} of data to generate robust estimates under adversarial feedback. Our algorithm, \textbf{Batch Robust Asynchronous \(Q\)-Learning}, abbreviated as \textcolor{winered}{\texttt{BR-Async-Q}}, is formally described in Algorithm~\ref{algo:epoch_joint}. At a high level, \textcolor{winered}{\texttt{BR-Async-Q}} partitions the interaction horizon \(T\) into \(K\) epochs (batches) of length \(H\) each, with epoch \(k\) (where $k=0, 1, \ldots$) given by \(\mc I_k:=\{kH,\ldots,(k+1)H-1\}\). From the data collected within each epoch $k$, we then generate a robust empirical estimate $\widehat{\mc{T}}_k$ of the Bellman optimality operator $\mc{T}$ by invoking Algorithm~\ref{alg:trimmed_huber} as a sub-routine.  Crucially, our main insight here is that choosing the epoch length $H$ to be large enough leads to a low variance in the estimate $\widehat{\mc{T}}_k$. This turns out to be the key to achieving tight bounds. Using $\widehat{\mc{T}}_k$, a  \emph{single} robust update to the \(Q\)-table is made at the end of the $k$-th epoch. Later, we shall see that $K=\mc{O}(\log(T))$ epochs suffice. Thus, the essence of our approach lies \emph{in  making infrequent \(Q\)-table updates using robust, low-variance estimates of the Bellman operator}. We now provide the details of our approach.  

\begin{algorithm}[h]
\caption{\textcolor{winered}{\texttt{BR-Async-Q}}}
\label{algo:epoch_joint}
\begin{algorithmic}[1]
    \State \textbf{Input:} Step size $\alpha$,  corruption levels $[\varepsilon_{\mc Y},\varepsilon_{\mc R}]$, confidence level $\delta$, total iterations $T=KH$, epoch length \(H\), and number of epochs \(K\). 
    \State Initialize $Q_0=0$ 
    \For{epoch $k=0,\ldots,K-1$}
        \State Initialize $\{\mc D_k(s,a), \mc Y_k(s,a)\}=\emptyset$, \(\forall\) $(s,a)\in\mc S\times\mc A$.
        \For{$t\in\mc I_k$}
            \State Observe $(s_t,a_t)$ and feedback $(\widetilde r_t(s_t,a_t),\widetilde s_{t+1})$.
            \State {\texttt{Append}} $\mc D_k(s_t,a_t) \leftarrow \widetilde r_t(s_t,a_t)$, $\mc Y_k(s_t,a_t) \leftarrow {\max}_{a\in\mc A} Q_k(\widetilde s_{t+1},a)$.
        \EndFor
        \State {\texttt{Compute}} {robust estimates} following~\eqref{eq:matrix_format}.
        \State {\texttt{Compute}} $(\widehat{\mc T}_k Q_k)(s,a)$ via~\eqref{eq:empirical_operator_def} for all \((s,a)\).
        {\texttt{Update}} $Q_{k+1}$ as per~\eqref{eq:epoch_relaxed_update}.
    \EndFor
\end{algorithmic}
\end{algorithm}
$\bullet $ \textbf{Online Batching and Robust Bellman Operator Estimation.} Fix an epoch $k\in [K]$ with index set $\mc I_k$. During epoch $k$, \textcolor{winered}{\texttt{BR-Async-Q}} maintains two dynamic data-sets for each state-action pair $(s,a) \in \mc{S} \times \mc{A}$:
\begin{equation}\label{eq:epoch_R_def}\medmath{
\begin{aligned}
\mc D_k(s,a) 
&:= \big\{\,\widetilde r_t(s_t,a_t)\ :\ t\in\mc I_k,\ (s_t,a_t)=(s,a)\,\big\}, \\
\mc Y_k(s,a) 
&:= \big\{\,{\max}_{a'\in\mc A}Q_k(\widetilde s_{t+1},a')\ \hspace{-1mm}:\ \hspace{-1mm}t\in\mc I_k,\ \hspace{-1mm} (s_t,a_t)=(s,a)\,\big\}.
\end{aligned}}
\end{equation}

\begin{figure}[t]
\begin{center}
\includegraphics[scale=0.50]{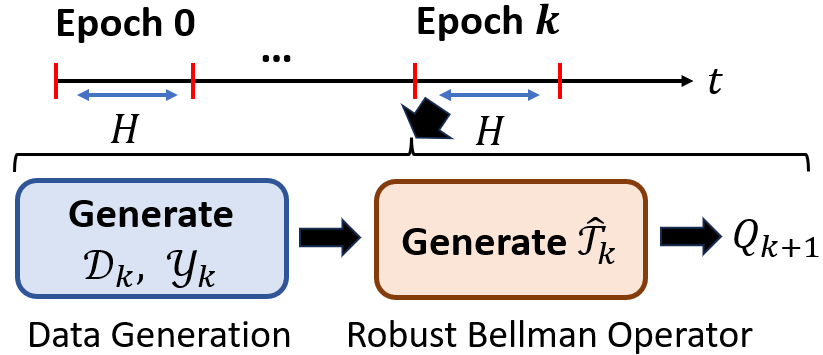}
\vspace{-4mm}
\end{center}
\caption{Illustration of \textcolor{winered}{\texttt{BR-Async-Q}} which runs in epochs of length $H$. During each epoch $k$, the learner generates the data sets $\mc{D}_k$ and $\mc{Y}_k$ in~\eqref{eq:epoch_R_def}, constructs a robust empirical estimate $\widehat {\mc{T}}_k$ of the Bellman optimality operator using such data, and then updates the \(Q\)-table to $Q_{k+1}$ using $\widehat {\mc{T}}_k$.}
\label{fig:Algo}
\vspace{-5mm}
\end{figure}
\noindent For each state-action pair $(s,a)\in\mc{S}\times\mc{A}$, let
\begin{equation}\label{eq:n-k} 
    N_k(s,a):=\sum_{t\in\mc{I}_k}\mathbf{1}\!\left\{(s_t,a_t)=(s,a)\right\}
\end{equation}
denote the number of visits to $(s,a) \in \mc{S} \times \mc{A}$ during epoch $k \in [K]$. Since the observations are collected asynchronously over a finite-length epoch, one may encounter degenerate scenarios where some state--action pairs are not visited during a given epoch. Consequently, for each pair $(s,a)$ and each epoch $k$, Algorithm~\ref{algo:epoch_joint} must account for the following two cases:  $N_k(s,a) > 0$ and $N_k(s,a) = 0$. 

\textcolor{winered}{\smallsquare}~\textbf{Case 1}~\((N_k(s,a)>0)\) \textbf{:} The first data-set $\mc D_k(s,a)$ contains the possibly corrupted reward observations for state-action pair $(s,a)$ collected during epoch $k$, where rewards are generated as per~\eqref{eqn:obs_corruption_model}. Similarly, the second dataset $\mc Y_k(s,a)$ contains the one-step look-ahead (bootstrapped) values evaluated using the \emph{frozen} iterate $Q_k$ (from the beginning of epoch $k$), and a next state observation generated as per~\eqref{eqn:obs_corruption_model}. Note that under our sampling model, where $s_t$ is generated independently from $\pi$ at each time-step, the clean samples in each of the above data sets are independent and identically distributed. Furthermore, the fractions of corrupted samples in $\mc D_k(s,a)$ and $\mc Y_k(s,a)$ are $\varepsilon_{\mc{R}}$ and $\varepsilon_{\mc{Y}}$, respectively, on average. Based on these observations, we use the trimmed-mean estimators \({\texttt{TRIM}}_{{\textup{H}}}\) and \({\textup{\texttt{TRIM}}}_{{\textup{B}}}\) from Section~\ref{sec:TrimmedMean} to estimate for every \((s,a)\in\mc S\times\mc A\), the clean reward mean
\(\medmath{R(s,a):=\mathbb{E}[r(s,a)]}\),
and the clean epoch-\(k\) bootstrap mean
\(\medmath{\mu_k(s,a):=\mathbb{E}_{s'\sim \mc P(\cdot\mid s,a)}\!\left[\max_{a'\in\mc A}Q_k(s',a')\right]}\) as follows: 
\begin{equation}\label{eq:matrix_format}
\begin{aligned}
\widehat R_k(s,a) &\leftarrow {\texttt{TRIM}}_{{\textup{H}}}~\![\mc D_k(s,a),\,\varepsilon_{\mc R},\,\delta_1],\\
\widehat \mu_k(s,a) &\leftarrow {\textup{\texttt{TRIM}}}_{{\textup{B}}}~\![\mc Y_k(s,a), (-B, B)],
\end{aligned}
\end{equation}
where \(B = 3\mc{C}\tilde{\sigma}/(1-\gamma)\), $\delta_1 = \delta/(4|\mc{S}||\mc{A}|T)$, and \(\mc{C}\) is the universal constant in~\eqref{eq:trim_var_bound}; the choice of \(B\) is justified in Lemma~\ref{lem:bounded}. Had we not used this threshold $B$ in the trimming operation for estimating $\widehat \mu_k(s,a)$, our final bounds would have featured a $O(\sqrt{\varepsilon_{\mc{Y}}})$ term. Using the threshold information allows us to instead obtain a sharper dependence of $O({\varepsilon_{\mc{Y}}})$ in our main convergence bound in~\eqref{eq:main-final-complete}. 

\textcolor{winered}{\smallsquare}~\textbf{Case 2}~\((N_k(s,a) = 0)\) \textbf{:} If a state--action pair $(s,a) \in \mc{S} \times \mc{A}$ is not visited during epoch $k \in [K]$, i.e., if $N_k(s,a)=0$, then the corresponding datasets $\mathcal{D}_k(s,a)$ and $\mathcal{Y}_k(s,a)$ are empty, and hence, the robust estimators in~\eqref{eq:matrix_format} cannot be evaluated. In this case, we set the corresponding estimates to zero by convention, i.e., $\widehat{R}_k(s,a)=\widehat{\mu}_k(s,a)=0$.

$\bullet$ \textbf{Stabilizing the Update via Clipping.}
Although the robust estimators in~\eqref{eq:matrix_format} control the reward and lookahead targets with high probability, they can still fail on rare events. In particular, a rare failure of the robust estimator can yield an extreme value of $\widehat R_k(s,a)$ or $\widehat\mu_k(s,a)$ defined in~\eqref{eq:matrix_format}, which, in turn, may destabilize the \(Q\)-table iterates. To guard against these low-probability but high-impact deviations, we apply a simple clipping step that uniformly bounds the empirical Bellman operator, ensuring that updates remain bounded even on extreme events. For this, our clipping threshold will be guided by the guarantees from Theorem~\ref{thm:trim_huber_both}. However, as apparent from the statement of this theorem, we need the number of samples fed to the robust estimator in Algorithm~\ref{alg:trimmed_huber} to exceed a minimum amount. Translated to the estimators in~\eqref{eq:matrix_format}, we need enough visits to each state-action pair $(s,a) \in \mc{S} \times \mc{A}$ in each epoch $k$. Accordingly, in Lemma~\ref{lem:counts_all_epochs_TAC_final}, we will show that as long as
\begin{equation}\label{eqn:condition_1}
H \ \ge \Big\lceil \left(20/\lambda_{\texttt{min}}\right)\cdot\log~\left(8 T|\mc S||\mc A|/\delta\right)\Big\rceil,
\end{equation}
a standard Bernstein bound ensures that in each epoch $k$, the number of visits to each state-action pair $(s,a)$ is at least $(1/2) \lambda_{\texttt{min}}H$, with high probability. Recall here that $\lambda_{\texttt{min}} >0$ is the minimum state-action pair visitation probability. On this event, for each pair $(s,a)$ and epoch $k$, $N_k(s,a) = |\mc D_k(s,a)| = |\mc Y_k(s,a)| \geq (1/2)  \lambda_{\texttt{min}}H$, which is exactly the regime in which Theorem~\ref{thm:trim_huber_both} provides meaningful high-probability control. For epoch length $H$ satisfying~\eqref{eqn:condition_1}, informed by the bound~\eqref{eq:trim_var_bound} from Theorem~\ref{thm:trim_huber_both}, we construct a clipping threshold as follows: 
\begin{equation}\label{eqn:Gt}
\texttt{G}_k \;:=\; \mathcal{C}\,\bar{\sigma}\left(\sqrt{\frac{2\log(8/\delta_1)}{\,\lambda_{\texttt{min}}\,H}}+\sqrt{{\varepsilon_{\mc R}}}\right)+\tilde{\sigma},
\end{equation}
where recall that $\tilde{\sigma}:=\max\{\bar{R},\bar{\sigma}\}$, $\mathcal{C}$ is the universal constant from Theorem~\ref{thm:trim_huber_both}, and $\delta_1 = \delta/(4|\mc{S}||\mc{A}|T)$. Next, define the clipping operator as ${\texttt{clip}}_{[{-\texttt{G}_k,\texttt{G}_k}]}(x):=\min\{ \max\{x,{-\texttt{G}_k}\},\,{\texttt{G}_k}\}, \forall x \in \mathbb{R}.$ Based on the estimates in~\eqref{eq:matrix_format} and the clipping radius in~\eqref{eqn:Gt}, we construct a robust empirical estimate $\widehat{\mc T}_k$ of the Bellman optimality operator $\mc{T}$ at the end of each epoch $k$, as follows: 
\begin{equation}\label{eq:empirical_operator_def}
(\widehat{\mc T}_k Q_k)(s,a)\hspace{-0.8mm}:={\texttt{clip}}_{[{-\texttt{G}_k,\texttt{G}_k}]}\left[\widehat R_k(s,a)\right]
\hspace{-1mm}+\hspace{-0.8mm}\gamma\,\widehat\mu_k(s,a),
\end{equation}
$\forall (s,a) \in \mc{S} \times \mc{A}$. The \(Q\)-table at the end of epoch $k$ is then updated for each \((s,a) \in \mc{S} \times \mc{A}\) as 
\begin{equation}
Q_{k+1}(s,a)
= 
(1-\alpha)Q_k(s,a)+\alpha\,\widehat{\mc T}_k Q_k(s,a).
\label{eq:epoch_relaxed_update}
\end{equation}

Clipping \(\widehat R_k(s,a)\) at the threshold in~\eqref{eqn:Gt} serves a two-fold purpose. First, it controls the magnitudes of the rewards used for updating $Q_k$, protecting against heavy-tailed reward distributions and reward corruptions. Second, it also indirectly controls the look-ahead estimate \(\widehat\mu_k(s,a)\), since the latter is computed from samples of the form \({\max}_{a'\in\mc A}Q_k(\widetilde s_{t+1},a')\), whose size is determined by \(\|Q_k\|_\infty\). As Lemma~\ref{lem:bounded} reveals, clipping the rewards ensures that \(\|Q_k\|_\infty\) remains uniformly bounded. This completes our description of \textcolor{winered}{\texttt{BR-Async-Q}}.

\begin{remark}
It is instructive to note that the ``asynchronous" aspect in our setup comes from the sampling model, where only one state-action pair is observed per time-step, as opposed to a synchronous sampling model~\cite{kearns, even2003learning, sidford} which assumes that all state-action pairs can be accessed at all times. That said, from~\eqref{eq:epoch_relaxed_update}, notice that although data arrives asynchronously, updates to the Q-table are made synchronously, once the learner has seen each state-action pair sufficiently often within each epoch. This strategy of acquiring enough samples before making updates is precisely what allows us to obtain order-optimal bounds.
\end{remark}

\section{Main Result and Discussion}
\label{sec:main_result}
In this section, we provide our main convergence result for  \textcolor{winered}{\texttt{BR-Async-Q}}. Define \(Q_k := [Q_k (s,a)]_{(s,a) \in \mc{S} \times \mc{A}}\), and $ e_k := Q_k - Q^*, \forall k \geq 0$. We then have the following result.
\label{sec:result}
\begin{theorem}
\label{thm:main theorem 1}
Suppose Assumption~\ref{ass:ergodic} holds. Given $\delta\in(0,1)$ and $T$, set the number of epochs  to  $\medmath{K = \left\lceil 2\log T/(1-\gamma)\right\rceil}$, and suppose $\medmath{T \ge K \left\lceil (64/\lambda_{\texttt{min}})  \log\!\left(8 T|\mc S||\mc A|/\delta\right) \right\rceil}$. Then, the output of Algorithm~\ref{algo:epoch_joint} with step-size $\medmath{\alpha:=\log T/((1-\gamma)K)}$ satisfies the following bound with probability at least $1-\delta$:
\begin{equation}
\label{eq:main-final-complete}\medmath{
\begin{aligned}
\lVert e_K \rVert_{\infty}
\ \le\
\frac{\lVert e_0 \rVert_{\infty}}{T}
\;&+\;
\mc{{O}}\left(\frac{\tilde\sigma}{(1-\gamma)^{\frac{5}{2}}}\right)\left(
\sqrt{\frac{ \log(T) \, \log\!\left(T|\mc S||\mc A|/\delta\right)}{\lambda_{\min}T}}
\right)\\
&+ \underbrace{\left[\mc{O}\left(\tilde{\sigma}\right)\left(
\frac{\textcolor{winered}{\varepsilon_{\mc Y}}}{(1-\gamma)^2}\right) \vee \mc{O}\left(\bar{\sigma}\right)\left(\frac{\sqrt{\textcolor{winered}{\varepsilon_{\mc R}}}}{1-\gamma}
\right)\right]}_{\psi},\end{aligned}}
\end{equation}where \(a \vee b = {\max}\left(a,b\right)\).
\end{theorem}
\textbf{Discussion of Theorem~\ref{thm:main theorem 1}.}
Theorem~\ref{thm:main theorem 1} gives a finite-time, high-probability \(\ell_\infty\)-error bound for Algorithm~\ref{algo:epoch_joint}. The bound in~\eqref{eq:main-final-complete} has two key components: a dominant statistical error term capturing nominal behavior, and a corruption-induced bias term. We discuss each of them below. 

$\bullet$  \textbf{Statistical term.}
Up to logarithmic factors, the dominant statistical term in~\eqref{eq:main-final-complete} scales as
\(
\tilde{\mc O}\!\left(\tilde\sigma(1-\gamma)^{-5/2}/\sqrt{\lambda_{\min}T}\right).
\) While our dependence on $\gamma$ matches existing bounds~\cite{Waiwright, Qu}, our dependence on the minimum visitation probability $\lambda_{\min}$ is, in fact, an improvement over that in~\cite{Qu} where the corresponding rate is $\tilde{\mc O}\!\left(\tilde\sigma(1-\gamma)^{-5/2}/(\lambda_{\min} \sqrt{T})\right)$. This improvement can be attributed to our batching strategy which leads to synchronous \emph{variance-reduced} updates.  

$\bullet$ \textbf{Corruption-induced bias.} The non-vanishing term $\psi$ in~\eqref{eq:main-final-complete} captures the effect of corruption-induced bias. Crucially, this term does not scale with the magnitude of corruption, but rather only the small corruption probabilities $\varepsilon_{\mc{Y}}$ and $\varepsilon_{\mc{R}}$. The significance of Theorem~\ref{thm:main theorem 1} is two-fold. It is the first result to show that $Q^*$ can be accurately estimated up to a small bias term, despite joint corruption of state and reward observations. Second, when only rewards are corrupted, i.e., when $\varepsilon_{\mc{Y}}=0$, the bias term $\mc{O}\left(\bar{\sigma}\right)\left(\frac{\sqrt{\textcolor{winered}{\varepsilon_{\mc R}}}}{1-\gamma}
\right)$ exactly matches the fundamental lower bound derived in~\cite{maity2025corruption}, sharpening the previous known best bounds under reward corruption. Specifically, the corruption-induced term in~\cite{maity2025corruption} scales inversely with $\lambda_{\min}$, suggesting that infrequent visitation of certain state-action pairs can amplify the effects of corruption. Our current work is the first to reveal that via variance-reduction, such an amplification can be avoided under the Huber model, where the adversary can only control \emph{what} to inject, but not \emph{when} to inject corruption signals; the latter is controlled by the exogenous Bernoulli processes defined in~\eqref{eqn:obs_corruption_model}. Overall, our main findings contribute to a finer understanding of \(Q\)-learning under adversarial feedback. 
\section{Simulation Results}
We evaluate \textcolor{winered}{\texttt{BR-Async-Q}} on a classic grid-world environment. The environment is an MDP with $|\mc S|=100$ states, $|\mc A|=40$ actions, and discount factor $\gamma=0.5$. The mean reward upper bound for all $(s,a)\in\mc S\times\mc A$ satisfies $\bar{R}\in(0,10]$, and the reward variances are uniformly bounded as $\bar{\sigma}^2\le 10$. Setting \(\varepsilon = \varepsilon_{\mc{Y}} = \varepsilon_{\mc{R}}\), we compare the performance of our proposed algorithm with vanilla \(Q\)-learning in Fig.~\ref{fig:sim}. The plots reveal the vulnerability of standard \(Q\)-learning to corrupted feedback, and show that \textcolor{winered}{\texttt{BR-Async-Q}} continues to guarantee convergence to a neighborhood of $Q^*.$ 

\begin{figure}[h]
\begin{center}
\begin{tabular}{cc}
   \hspace{-5 mm}\includegraphics[scale=0.22]{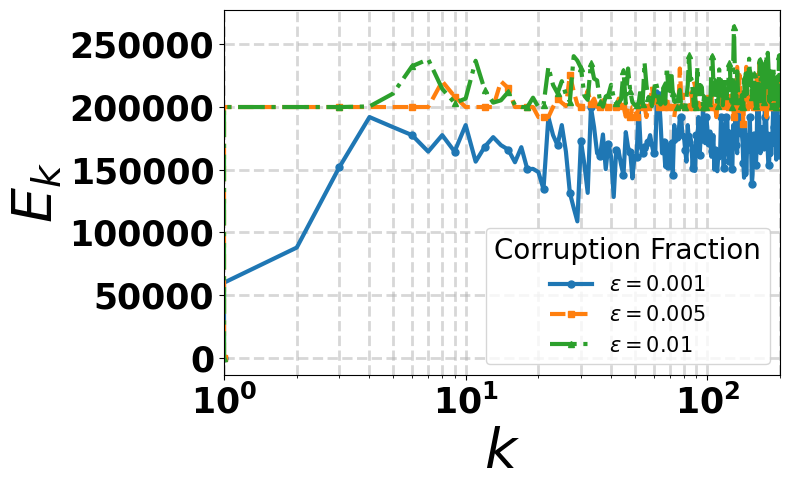}&\hspace{-4 mm}\includegraphics[scale=0.22]{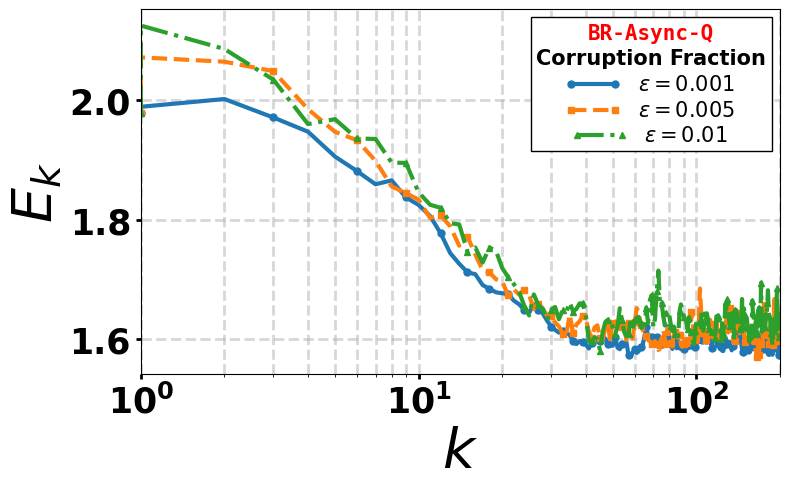}
    \end{tabular}
    \begin{tabular}{cc}
   \hspace{-5 mm}\includegraphics[scale=0.22]{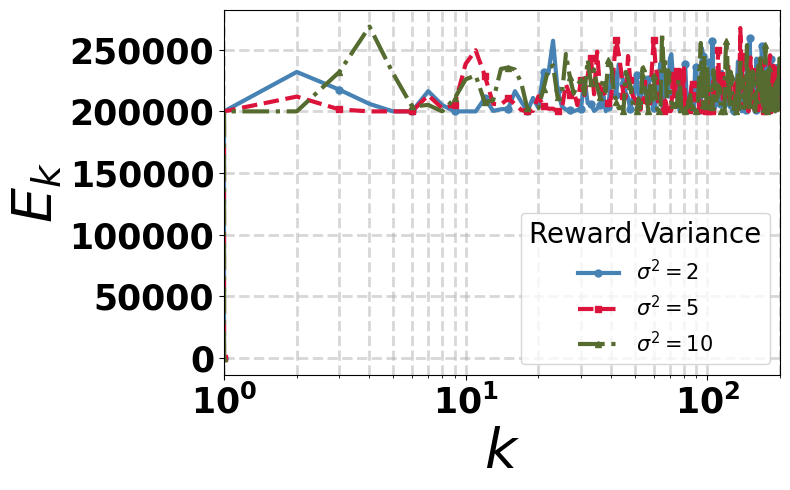}&\hspace{-4 mm}\includegraphics[scale=0.22]{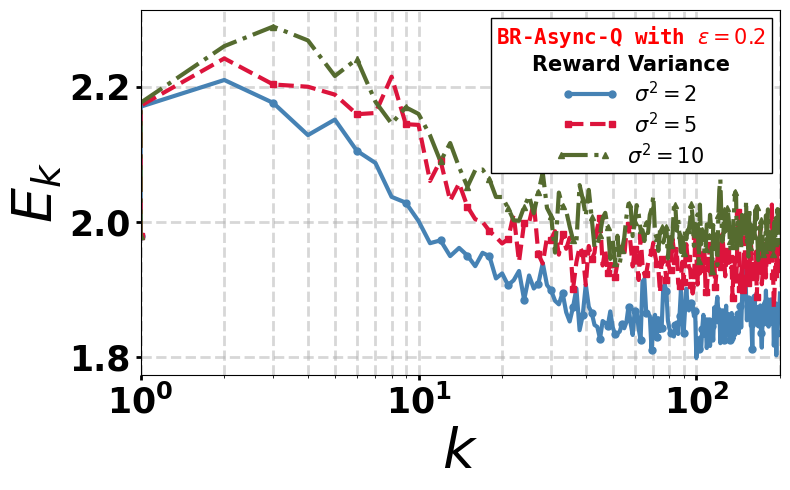}
    \end{tabular}
\vspace{-4mm}
\end{center}
\caption{\textbf{(Top Left)} $\ell_\infty$ error $E_k=\lVert Q_k - Q^\star \rVert_\infty$ for \texttt{\textcolor{mygreen}{Vanilla-Q}} under the Huber-contaminated model in~\eqref{eqn:obs_corruption_model}, with $\varepsilon\in\{0.001,0.005,0.01\}$, reward-noise variance $\sigma^2=5$, a $-10^6$ biasing attack on corrupted reward samples, and next-state corruption in which observed next states are replaced by arbitrary states, modeled by uniform resampling over $\mc S$. \textbf{(Top Right)} $E_k$ for \textcolor{winered}{\texttt{BR-Async-Q}} under the same corruption levels, noise statistics, and attack. \textbf{(Bottom Left)} $E_k$ for \textcolor{mygreen}{\texttt{Vanilla-Q}} with $\varepsilon=0.2$ and reward-noise variance $\sigma^2\in\{2,5,10\}$. \textbf{(Bottom Right)} $E_k$ for \textcolor{winered}{\texttt{BR-Async-Q}} with $\varepsilon=0.2$ and reward-noise variance $\sigma^2\in\{2,5,10\}$. Each curve in Fig.~\ref{fig:sim} reports the average over $100$ independent runs.}
\label{fig:sim}
\end{figure}
\vspace{-2 mm}
We next consider the same MDP with asymmetric next-state and reward corruption rates $(\varepsilon_{\mc Y},\varepsilon_{\mc R})$ i.e., \(\varepsilon_{\mc Y} \neq \varepsilon_{\mc R}\). Corrupted rewards are replaced by $-10^8$, while corrupted next states are sampled uniformly from $\mc S$. As shown in Fig.~\ref{fig:sim-1}, vanilla \(Q\)-learning is highly vulnerable to such corruption, whereas \textcolor{winered}{\texttt{BR-Async-Q}} remains stable and converges to a neighborhood of $Q^\star$.
\begin{figure}[h]
\begin{center}
\begin{tabular}{cc}
   \hspace{-5 mm}\includegraphics[scale=0.25]{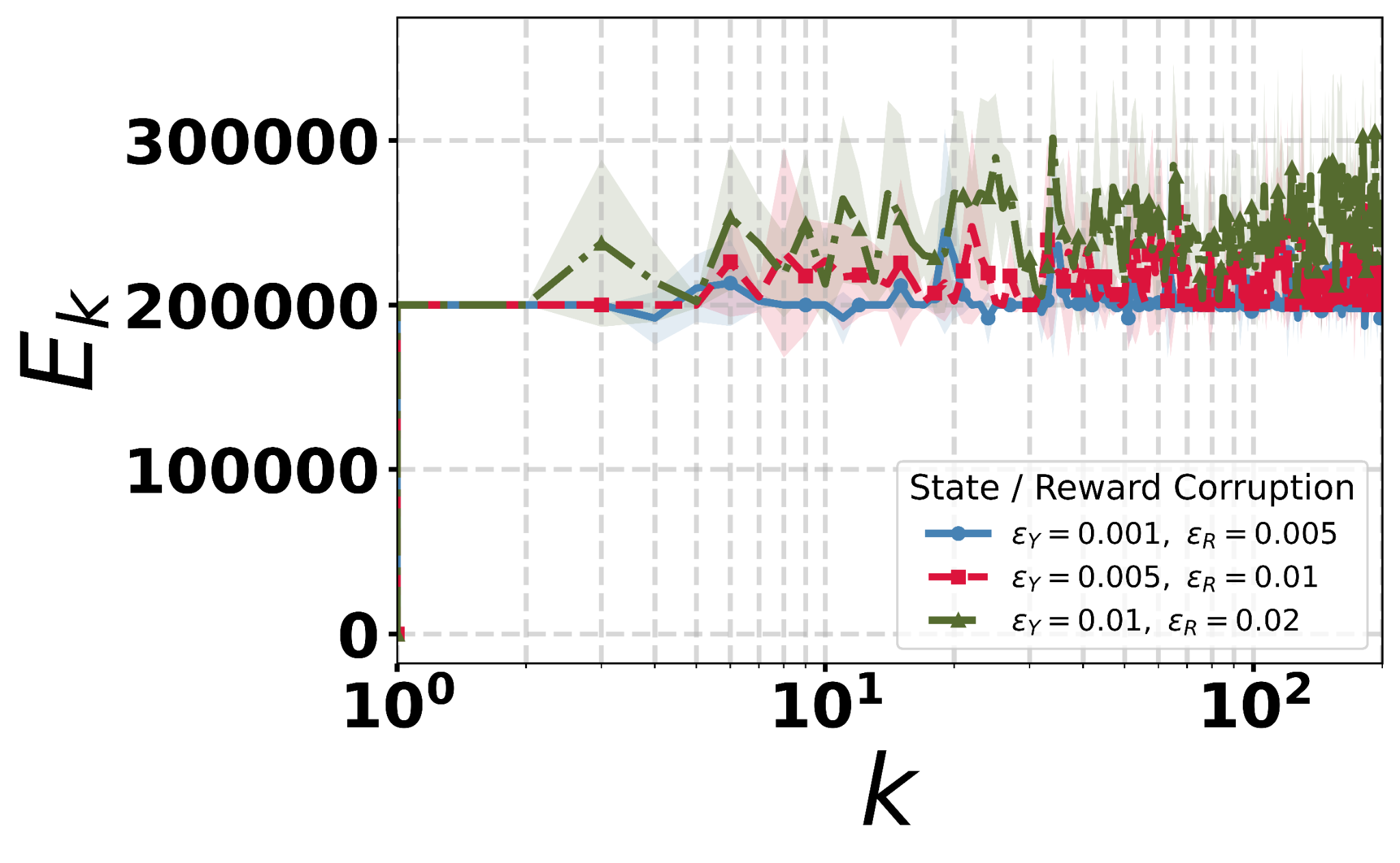}&\hspace{-4 mm}\includegraphics[scale=0.25]{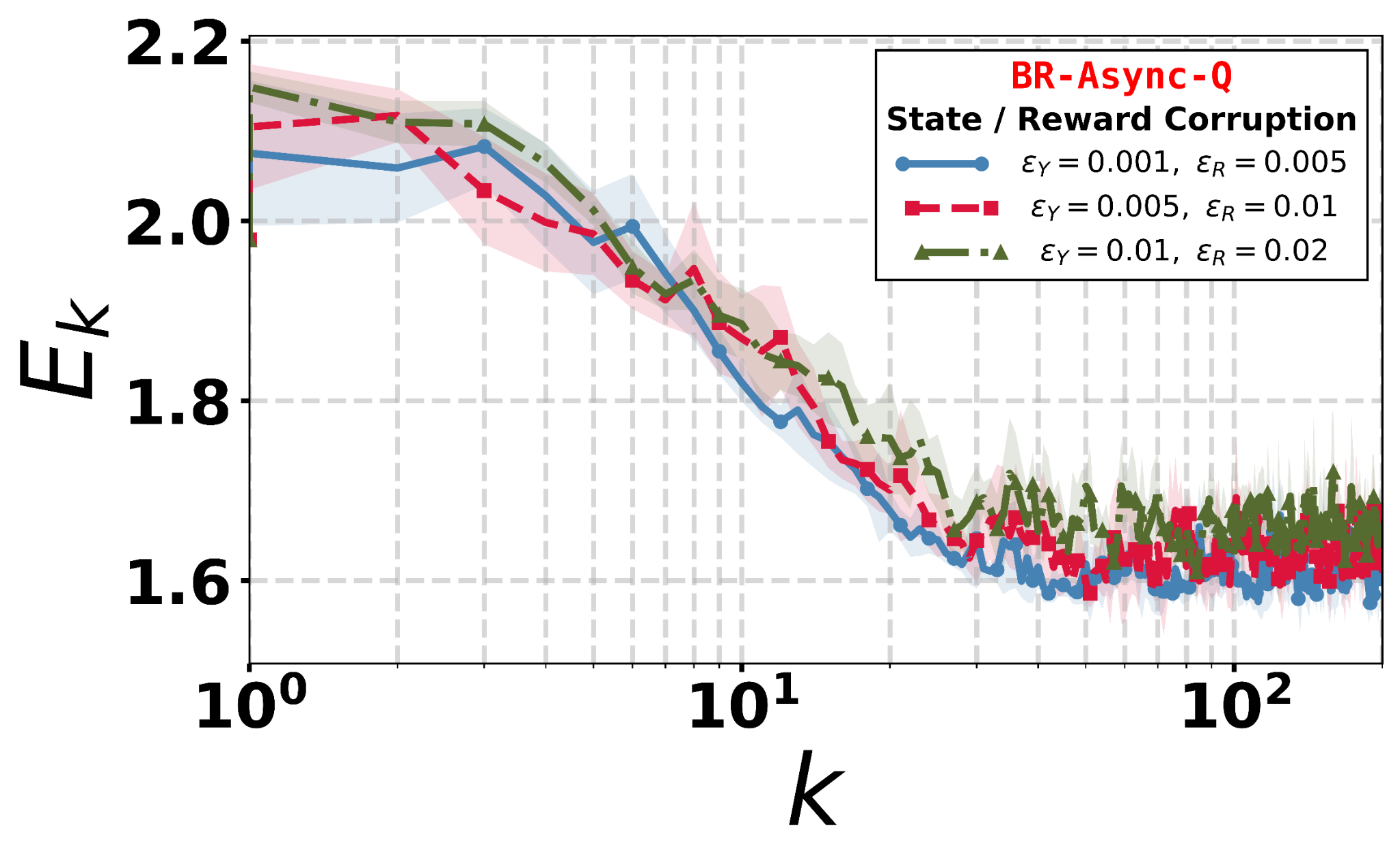}
    \end{tabular}
\vspace{-4mm}
\end{center}
\caption{\textbf{(Left)} $\ell_\infty$ error $E_k=\lVert Q_k-Q^\star\rVert_\infty$ for
\textcolor{mygreen}{\texttt{Vanilla-Q}} under the Huber-contaminated model in~\eqref{eqn:obs_corruption_model}, with distinct next-state and reward corruption fractions
$(\varepsilon_Y,\varepsilon_R)\in \{(0.001,0.005),(0.005,0.01),(0.01,0.02)\}$. Corrupted reward samples are subjected to a $-10^8$ biasing attack, while corrupted next-state observations are replaced by states sampled uniformly from $\mc S$. \textbf{(Right)} The corresponding error $E_k$ for \textcolor{winered}{\texttt{BR-Async-Q}} under the same corruption levels and attack model. Each curve in Figure~\ref{fig:sim-1} reports the average over 100 independent runs, and the shaded region represents one standard deviation around the mean.}
\label{fig:sim-1}
\end{figure}

Lastly, we examine the effect of the epoch length \(H\) on \textcolor{winered}{\texttt{BR-Async-Q}} under two asymmetric corruption settings; see Figure~\ref{fig:sim-3}.
\begin{figure}[h]
\begin{center}
\begin{tabular}{cc}
   \hspace{-5 mm}\includegraphics[scale=0.22]{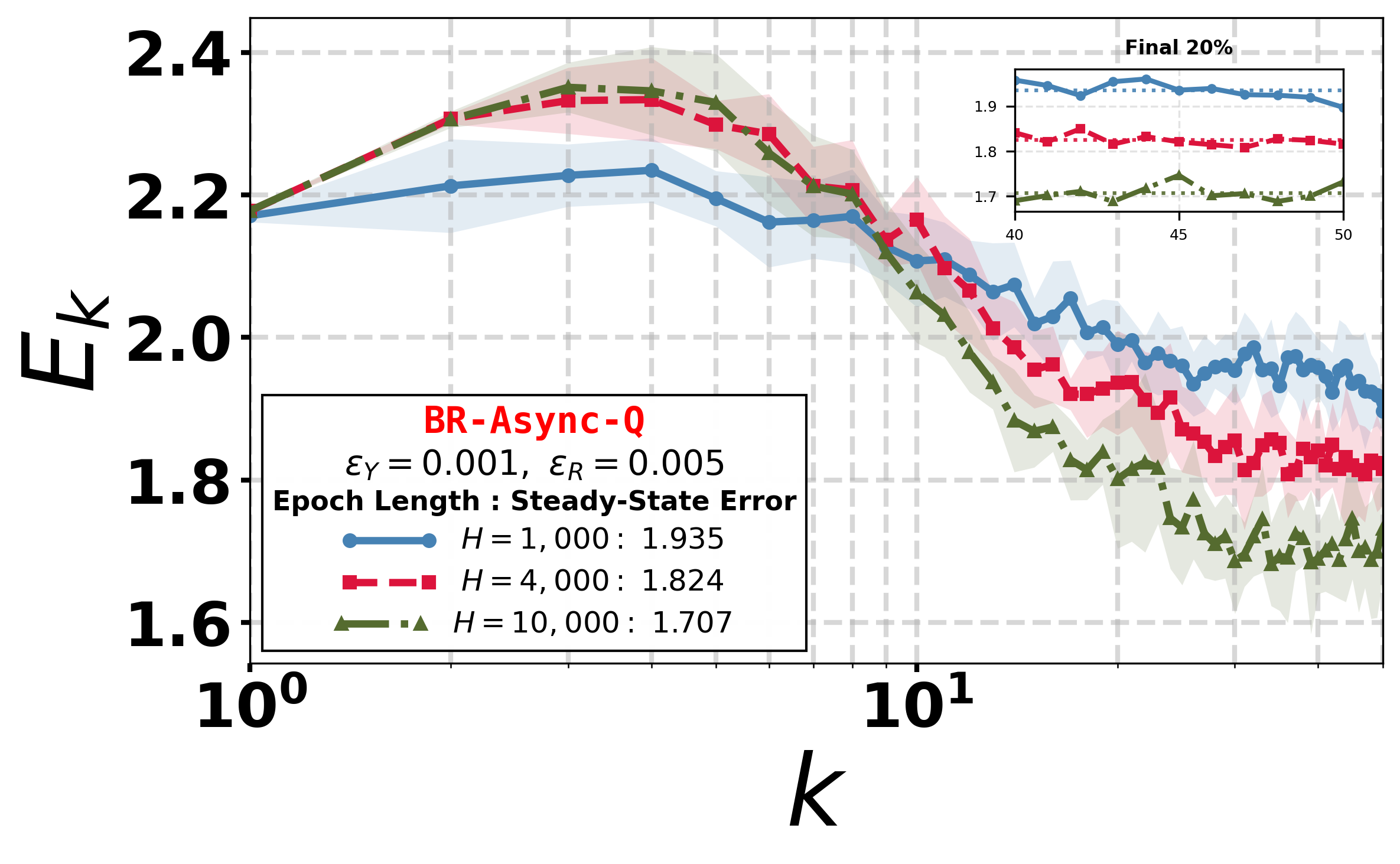}&\hspace{-4 mm}\includegraphics[scale=0.22]{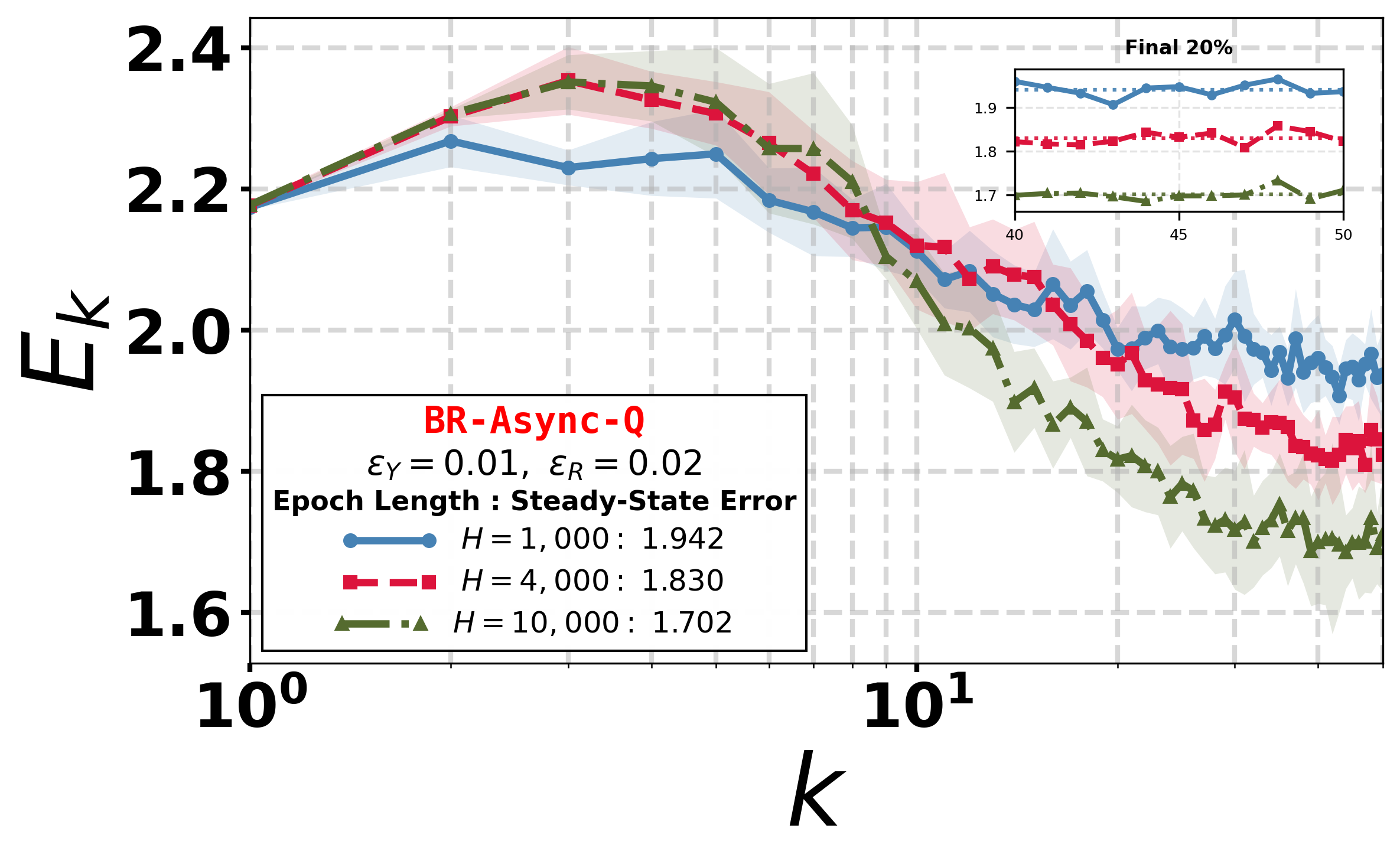}
    \end{tabular}
\vspace{-4mm}
\end{center}
\caption{\textbf{(Left)} The $\ell_\infty$ error $E_k=\lVert Q_k-Q^\star\rVert_\infty$ for \textcolor{winered}{\texttt{BR-Async-Q}} under the asymmetric corruption levels $(\varepsilon_Y,\varepsilon_R)=(0.001,0.005)$, with epoch length $H\in\{1000,4000,10000\}$ and a fixed number of epochs $K=50$. \textbf{(Right)} The corresponding error under the higher corruption levels $(\varepsilon_Y,\varepsilon_R)=(0.01,0.02)$. In both panels, corrupted rewards are replaced by $-10^6$, corrupted next states are sampled uniformly from $\mc S$, and the reward-noise variance is $\sigma^2=5$. Each curve in Figure~\ref{fig:sim-3} reports the average over 50 independent runs, and the shaded region represents one standard deviation around the mean.}
\label{fig:sim-3}
\end{figure}

\section{Main Convergence Analysis}
In this section, we provide the detailed proof of Theorem~\ref{thm:main theorem 1}. Essentially, this boils down to controlling the error between the estimated Bellman operator $\widehat{\mc T}_k$ in~\eqref{eq:empirical_operator_def} and the true operator $\mc{T}$. To control this error, we will derive high-probability bounds on the reward estimation error $\Vert \widehat R_k -R \Vert_{\infty}$ in Lemma~\ref{lem:R_concentration_all} and look-ahead estimation error $\Vert \widehat \mu_k - \mu \Vert_{\infty}$ in Lemma~\ref{lem:mu_concentration_all}. Here, $\widehat R_k, R, \widehat \mu_k$, and $\mu$ are $|\mc{S}| \times |\mc{A}|$ dimensional vectors with $\widehat R_k(s,a),  R(s,a), \widehat \mu_k(s,a)$, and $\mu(s,a)$ as their respective $(s,a)$-th components. Our first step is to establish that the clipping operation in~\eqref{eq:empirical_operator_def} ensures that the sequence $\{Q_k\}$ remains uniformly bounded. 

\begin{lemma} 
\label{lem:bounded}
The following bound holds for all $k\ge 0$:
\begin{equation}
\label{eq:bounded_iterates_br_async_q}
\norm{Q_k}_\infty \ \le\ B:= \frac{3 \mc{C} \tilde{\sigma}}{1-\gamma},
\end{equation}
where $\mc{C}$ is the universal constant from Theorem~\ref{thm:trim_huber_both}.
\end{lemma}
\begin{proof}
Recalling the definition of \(N_k(s,a)\) in~\eqref{eq:n-k}, we prove the result by induction on \(k\), treating separately the cases \(N_k(s,a)>0\) and \(N_k(s,a)=0\).

The base case holds since $Q_0=0$. For the inductive step, suppose that
$\norm{Q_k}_\infty\le B$. It then follows that \(\left|\max_{a\in\mc{A}}Q_k(s,a)\right|\le B,\quad\forall s\in\mc{S}.\) Consequently, every look-ahead sample $\max_{a'\in\mc{A}}Q_k(\widetilde{s}_{t+1},a')$ lies in $[-B,B]$.

We now bound the estimators for an arbitrary
$(s,a)\in\mc{S}\times\mc{A}$ by considering two cases. If
$N_k(s,a)>0$, then
${\texttt{TRIM}}_{\textup{B}}$ in~\eqref{eq:matrix_format} averages
only values in $[-B,B]$, and hence \(\left|\widehat{\mu}_k(s,a)\right|\le B.\)
Moreover, by the definition of the clipping operator,
\(\left|{\texttt{clip}}_{[-\texttt{G}_k,\texttt{G}_k]}\bigl(\widehat{R}_k(s,a)\bigr)\right|\le \texttt{G}_k.\)
If $N_k(s,a)=0$, then, by convention,
$\widehat{\mu}_k(s,a)=\widehat{R}_k(s,a)=0$, so the same two bounds
hold. Therefore, uniformly over all $(s,a)\in\mc{S}\times\mc{A}$,
\(
\left|\widehat{\mu}_k(s,a)\right|\le B,
\quad
\left|
{\texttt{clip}}_{[-\texttt{G}_k,\texttt{G}_k]}
\bigl(\widehat{R}_k(s,a)\bigr)
\right|
\le \texttt{G}_k.
\)
It follows from~\eqref{eq:empirical_operator_def} that
\(\norm{\widehat{\mc T}_kQ_k}_\infty \le \texttt{G}_k+\gamma B.\)
Using the update rule in~\eqref{eq:epoch_relaxed_update}, we obtain
\( \norm{Q_{k+1}}_\infty \le (1-\alpha)B+\alpha\bigl(\texttt{G}_k+\gamma B\bigr).\)
By the definition of $\texttt{G}_k$ in~\eqref{eqn:Gt}, whenever the
epoch length $H$ satisfies~\eqref{eqn:condition_1}, we have
$\texttt{G}_k\le 3\mc{C}\tilde{\sigma}$. Hence,
\( \norm{Q_{k+1}}_\infty \le \bigl(1-\alpha(1-\gamma)\bigr)B +3\alpha\mc{C}\tilde{\sigma} = B,\)
where the final equality follows from
$B=3\mc{C}\tilde{\sigma}/(1-\gamma)$. This completes the induction step and the proof.
\end{proof}
Our next step is to control the per-epoch look-ahead estimation error. This requires ensuring that within every epoch, we have adequate data for every state-action pair. However, in a given epoch $k$, the number of visits to any state-action pair $(s,a)$, denoted by \(N_k(s,a):=\sum_{t\in\mc I_k}\mathbf{1}\{(s_t,a_t)=(s,a)\}\), is random. Nonetheless, we next argue that as long as the epoch length $H$ is long enough, every state-action pair will be sufficiently visited on a high-probability event.  

\begin{lemma} 
\label{lem:counts_all_epochs_TAC_final}
Consider the following event:
\begin{equation}\label{eqn:count_final}\medmath{
\mc V := \bigcap_{k=0}^{K-1}\left\{\min_{(s,a)\in\mc S\times\mc A} N_k(s,a)\ge \tfrac12\lambda_{\min}H\right\}. 
}
\end{equation}
If $H$ satisfies the condition in~\eqref{eqn:condition_1}, then event $\mc{V}$ holds with probability at least $1-\delta_1$, where $\delta_1 = \delta/(4|\mc{S}||\mc{A}|T)$. 
\end{lemma}
\begin{proof}
Fix an epoch $k$ and a state-action pair $(s,a)$. Under our sampling model, $\{\mathbf{1}\{(s_t,a_t)=(s,a)\}\}_{t \in \mc{I}_k}$ is an i.i.d. sequence with mean $\lambda(s,a)$ and variance at most $\lambda(s,a).$ A simple application of Bernstein's inequality then tells us that
$$ \mathbb{P}\left(N_k(s,a) \leq 0.5 \lambda_{\texttt{min}} H\right) \leq 2 \exp{\left(-\frac{3}{28} H \lambda(s,a)\right)}.$$
Requiring the R.H.S. of the above display to be smaller than $\delta_1$, and then union-bounding over all state-action pairs and $K \leq T$ epochs, leads to the desired claim; the details are identical to those in~\cite[Lemma~6]{maity2025corruption}.
\end{proof}

We can now control the look-ahead estimation error. 
\begin{lemma} (\textbf{Look-ahead Estimation Error})
\label{lem:mu_concentration_all} Suppose the conditions of Theorem~\ref{thm:main theorem 1} hold. Then, with probability at least \(1-\delta/2\), the following bound holds simultaneously for all epochs $k=0, \ldots, K-1$:
\begin{equation}\label{eq:mu_bound}\medmath{
\|\hat{\mu}_k-\mu_k\|_{\infty}
\ \le\ \Delta_1:= 
\,\mc{O}\left(\frac{\tilde{\sigma}}{1-\gamma}\right)
\left(
\sqrt{\frac{\log\!\left(T|\mc S||\mc A|/\delta\right)}{\lambda_{\min}H}}
\;+\;
{\varepsilon_{\mc Y}}
\right).}
\nonumber
\end{equation}
\end{lemma}

\begin{proof}
Fix an epoch $k$ and a state-action pair $(s,a)$. Our goal is to bound the error $|\widehat \mu_k(s,a) - \mu_k(s,a)|$, where $\widehat \mu_k(s,a) = {\textup{\texttt{TRIM}}}_{{\textup{B}}}~\![\mc Y_k(s,a),\, (-B, B)]$, and $\mu_k(s,a) = \mathbb{E}_{s'\sim \mc P(\cdot\mid s,a)}\!\left[\max_{a'\in\mc A}Q_k(s',a')\right].$ To do so, we wish to leverage Theorem~\ref{thm:trim_huber_both}. However, the main difficulty arises from the fact that while the number of samples $M$ in Theorem~\ref{thm:trim_huber_both} is deterministic, $|\mc Y_k(s,a)|= N_k(s,a)$ is a random variable. To work around this issue, we start by defining a random-threshold deviation event as follows:
\begin{equation}
\label{eq:Ek1_overleaf}\medmath{\mc E_{k,1}(s,a):=\hspace{-1mm}
\left\{
\left|\widehat\mu_k(s,a)-\mu_k(s,a)\right|
\le \hspace{-1mm}
\,\frac{6 \bar{c} \, \mc{C} \tilde{\sigma}}{1-\gamma}\hspace{-1mm}\left(
\sqrt{\frac{\log(8/\delta_1)}{N_k(s,a)}}+{\varepsilon_{\mc{Y}}}
\right)
\right\},}
\nonumber 
\end{equation}
where $\bar{c}$ and $\mc C$ are the universal constants that appear in~\eqref{eq:trim_bounded_refined} and~\eqref{eq:trim_var_bound}, respectively. We define a corresponding deterministic-threshold version of this event as
\begin{equation}
\label{eq:Ek_overleaf}
\medmath{\mc E_{k}(s,a):=\hspace{-1mm}
\left\{
\left|\widehat\mu_k(s,a)-\mu_k(s,a)\right|
\le \hspace{-1mm}
\,\frac{6 \bar{c} \, \mc{C} \tilde{\sigma}}{1-\gamma}\hspace{-1mm}\left(
\sqrt{\frac{2\log(8/\delta_1)}{\lambda_{\texttt{min}}H}}+{\varepsilon_{\mc{Y}}}
\right)
\right\}.}
\nonumber 
\end{equation}
The above events are informed by the fact that each sample in $\mc{Y}_k(s,a)$ is corrupted with probability $\varepsilon_{\mc{Y}}$. 
Our immediate goal is to control the failure probability of the ``good" event $\mc E_k(s,a).$ To set the stage for this, let $\mc{F}_k$ be the sigma-field generated by all the randomness up to the beginning of epoch $k$. By construction, $Q_k$ is $\mc{F}_k$-measurable. Conditioned on $\mc{F}_k$ then, the only randomness that remains in the uncorrupted samples of the form $\max_{a' \in \mc{A}} Q_k(s_{t+1}, a')$ in $\mc{Y}_k(s,a)$ comes from the next state observation $s_{t+1} \sim \mc{P}(\cdot| s,a).$ Under our i.i.d. data generation assumption, all the clean samples in $\mc{Y}_k(s,a)$ are i.i.d. and have mean $\mu_k(s,a)$ conditioned on $\mc{F}_k.$ Moreover, in light of Lemma~\ref{lem:bounded}, each clean sample in $\mc{Y}_k(s,a)$ is contained in the interval $[-B, B],$ where $B={3 \mc{C} \tilde{\sigma}}/{(1-\gamma)}.$ For any fixed $j \geq 0.5 \lambda_{\texttt{min}} H$, appealing to~\eqref{eq:trim_bounded_refined} in Theorem~\ref{thm:trim_huber_both}  then yields 
\begin{equation}\medmath{
\beta:=\mathbb{P}~\!\left(\mc E_{k,1}(s,a)^{\texttt C}\,\big|\,\mc F_k,\ N_k(s,a)=j\right)\ \le\ \delta_1.}
\end{equation}
We next remove the conditioning on $\mc F_k$ by using the tower property of expectations as follows: 
\begin{equation}\label{eqn:tower}\medmath{
\begin{aligned}
&\mathbb{P}~\!\left(\mc E_{k,1}(s,a)^{\texttt C}\,\big|\,N_k(s,a)=j\right) = \mathbb{E}\left[\mathbf{1}_{\mc E_{k,1}(s,a)^{\texttt C}}\big|N_k(s,a)=j\right]\\
&=\mathbb{E}\!\left[\mathbb{E}~\!\left[\mathbf{1}_{\mc E_{k,1}(s,a)^{\texttt C}}\,\big|\,\mc F_k,\ N_k(s,a)=j\right]\right] = \mathbb{E}[\beta] \le \delta_1.
\end{aligned}}
\end{equation}

To relate the above result back to our event of interest $\mc{E}_k(s,a)$, observe from the definitions of the events $\mc{E}_{k,1}(s,a)$ and $\mc{E}_k(s,a)$, that for any $j \geq 0.5 \lambda_{\texttt{min}} H$, the event $\Bigl\{\mc E_k(s,a)^{\texttt C}\mid N_k(s,a)=j\Bigr\}$ 
 implies the event  
$\Bigl\{\mc E_{k,1}(s,a)^{\texttt C}\mid N_k(s,a)=j\Bigr\}.$ As a result, 
\begin{equation}\label{eq:prob_ineq}\medmath{
\mathbb{P}~\!\Bigl(\mc E_k(s,a)^{\texttt C}\,\big|\, N_k(s,a)=j\Bigr)
\ \le\
\mathbb{P}~\!\Bigl(\mc E_{k,1}(s,a)^{\texttt C}\,\big|\, N_k(s,a)=j\Bigr) \leq \delta_1,}
\end{equation}
where the last inequality follows from~\eqref{eqn:tower}. To get rid of the conditioning w.r.t. $N_k(s,a)$ in the above probabilities, we will now leverage the event $\mc{V}$ from Lemma~\ref{lem:counts_all_epochs_TAC_final} where all state-action pairs are sufficiently visited. To that end, let $p:=\mathbb{P}~\bigl(\mc E_k(s,a)^{\texttt C}\cap \mc V\bigr)$ and observe that:
\begin{equation}
\label{eqn:prob_decomp}
\medmath{
\begin{aligned}
p &= \sum_{j=0}^{H} \mathbb{P}~\bigl(\mc E_k(s,a)^{\texttt C}\cap \mc V \cap \{N_k(s,a)=j\} \bigr)\\
&\overset{(a)}= \sum_{j=0.5 \lambda_{\texttt{min}}H}^{H} \mathbb{P}~\bigl(\mc E_k(s,a)^{\texttt C}\cap \mc V \cap \{N_k(s,a)=j\} \bigr)\\
& \leq \sum_{j=0.5 \lambda_{\texttt{min}}H}^{H} \mathbb{P}~\bigl(\mc E_k(s,a)^{\texttt C} \cap \{N_k(s,a)=j\} \bigr)\\
& = \sum_{j=0.5 \lambda_{\texttt{min}}H}^{H} \mathbb{P}~\bigl(\mc E_k(s,a)^{\texttt C} | N_k(s,a)=j \bigr) \mathbb{P}~\bigl(N_k(s,a)=j\bigr)\\
& \overset{(b)}\leq \delta_1 \sum_{j=0.5 \lambda_{\texttt{min}}H}^{H} \mathbb{P}~\bigl(N_k(s,a)=j \bigr) \leq \delta_1 \sum_{j=0}^{H} \mathbb{P}~\bigl(N_k(s,a)=j \bigr) = \delta_1.
\end{aligned}}
\end{equation}
For (a), we used $N_k(s,a) \geq 0.5 \lambda_{\texttt{min}}H$ on event $ \mc{V}$, and for (b), we used~\eqref{eq:prob_ineq}. Using~\eqref{eqn:prob_decomp} and the fact that $\mathbb{P}(\mc{V}^{\texttt C}) \leq \delta_1$ from Lemma~\ref{lem:counts_all_epochs_TAC_final}, we conclude that $\mathbb{P}~\bigl(\mc E_k(s,a)^{\texttt C}\bigr)  \leq
\mathbb{P}~\bigl(\mc E_k(s,a)^{\texttt C}\cap \mc V\bigr)\ +\ \mathbb{P}~(\mc V^{\texttt C})
\leq 2 \delta_1 \leq \delta/(2 |\mc{S}| |\mc{A}| T).$\footnote{Here, we used that for any two events $\mc A$ and $\mc B$, the following is true: 
\(\mathbb{P}~(\mc A)
=\mathbb{P}~\bigl(\mc A\cap \mc B\bigr)+\mathbb{P}~\bigl(\mc A\cap \mc B^{\texttt C}\bigr) \leq \mathbb{P}~\bigl(\mc A\cap \mc B\bigr) + \mathbb{P}~\bigl(\mc{B}^{\texttt{C}} \bigr)\).} We have thus established a failure probability bound for the ``good" event $\mc E_k(s,a)$ for a fixed state-action pair $(s,a)$ and a fixed epoch $k$. The claim of the lemma then follows by simply union-bounding over all state-action pairs and all epochs, and using $K \leq T.$ 
\end{proof}

We now establish a similar result for reward estimation. 

\begin{lemma}(\textbf{Reward Estimation Error})
\label{lem:R_concentration_all}
Suppose the conditions of Theorem~\ref{thm:main theorem 1} hold. Then, with probability at least \(1-\delta/2\), the following bound holds simultaneously for all epochs $k=0, \ldots, K-1$:
\begin{equation}\label{eq:R_bound}\medmath{
\|\hat{R}_k- R_k\|_{\infty}
\ \le\ \Delta_2:=
\,\mc{O}\left(\bar\sigma\right)
\left(
\sqrt{\frac{\log\!\left(T|\mc S||\mc A|/\delta\right)}{\lambda_{\min}H}}
\;+\;
\sqrt{{\varepsilon_{\mc R}}}
\right).}
\nonumber
\end{equation}
\end{lemma}

\begin{proof} The argument mirrors Lemma~\ref{lem:mu_concentration_all}, and hence we only highlight the essential points of departure. Like before, fix a state-action pair $(s,a)$ and an epoch $k$. This time, we wish to study the error $|\widehat R_k(s,a)-R(s,a)|$, where $\widehat R_k(s,a) = {\texttt{TRIM}}_{{\textup{H}}}~\![\mc D_k(s,a),\,\varepsilon_{\mc R},\,\delta_1]$, and $\mc D_k(s,a)$ is the data-set of noisy rewards for pair $(s,a)$ collected during epoch $k$, with $|\mc{D}_k(s,a)| = N_k(s,a).$ Similar to the proof of Lemma~\ref{lem:mu_concentration_all}, we define the following two events:
\begin{equation}\medmath{
\begin{aligned}
\mc J_{k,1}(s,a)
:&=\left\{|\widehat R_k(s,a)-R(s,a)|
\le \mc C\bar\sigma\!\left(\sqrt{\frac{\log(8/\delta_1)}{N_k(s,a)}}+\sqrt{{\varepsilon_{\mc R}}}\right)\right\}, \\
\mc J_k(s,a)
:&=\left\{|\widehat R_k(s,a)-R(s,a)|
\le \mc C\bar\sigma\!\left(\sqrt{\frac{2\log(8/\delta_1)}{\lambda_{\min}H}}+\sqrt{{\varepsilon_{\mc R}}}\right)\right\}.
\end{aligned}
}
\nonumber
\end{equation}
Following~\eqref{eqn:obs_corruption_model}, each sample in $\mc{D}_k(s,a)$ is corrupted with probability $\varepsilon_{\mc{R}}$, and the clean samples are independent with mean $R(s,a)$ and variance at most $\bar\sigma^2$. For any fixed $j \geq 0.5 \lambda_{\texttt{min}} H$, appealing to~\eqref{eq:trim_var_bound} in Theorem~\ref{thm:trim_huber_both} then yields
\begin{equation}
\label{eqn:prob1}
\mathbb{P}~\!\big(\mc J_{k,1}(s,a)^{\texttt C}\,\big|\,N_k(s,a)=j\big)\le \delta_1.
\end{equation}
Here, we used the fact that the choice of $H$ in~\eqref{eqn:condition_1} ensures  $\delta_1 \ge 8e^{-j/2}$ when $j \geq 0.5 \lambda_{\texttt{min}} H$, allowing us to invoke Theorem~\ref{thm:trim_huber_both}. Next, we observe that for $j \geq 0.5 \lambda_{\texttt{min}} H$, conditioned on $\{N_k(s,a) =j\},$ the event $\mc J_k(s,a)^{\texttt C}$ implies the event $\mc J_{k,1}(s,a)^{\texttt C}$. Thus, from~\eqref{eqn:prob1}, we obtain $\mathbb{P}~\!\big(\mc J_{k}(s,a)^{\texttt C}\,\big|\,N_k(s,a)=j\big)\le \delta_1$. Following an identical reasoning as in~\eqref{eqn:prob_decomp}, we can remove the conditioning on $\{N_k(s,a)=j\}$ to obtain $\mathbb{P}~\bigl(\mc J_k(s,a)^{\texttt C}\cap \mc V\bigr) \leq \delta_1$, implying $\mathbb{P}~\bigl(\mc J_k(s,a)^{\texttt C}\bigr)  \leq
\mathbb{P}~\bigl(\mc J_k(s,a)^{\texttt C}\cap \mc V\bigr)\ +\ \mathbb{P}~(\mc V^{\texttt C})
\leq 2 \delta_1 \leq \delta/(2 |\mc{S}| |\mc{A}| T).$ The rest follows from union-bounding. 
\end{proof}

We are now ready to complete the proof of Theorem~\ref{thm:main theorem 1}. 

\begin{proof} (\textbf{Proof of Theorem~\ref{thm:main theorem 1}}) 
Let $e_k:=Q_k-Q^\star$. Our strategy is to set up a recursion for $e_k$ in terms of the perturbation $\Vert \hat{\mc{T}}_k - \mc{T} \Vert_{\infty}$, and then use Lemmas~\ref{lem:mu_concentration_all} and~\ref{lem:R_concentration_all} to control the perturbation. Before doing so, we note that the Bellman optimality operator $\mc{T}$ is defined as follows:
\begin{equation}\label{eqn:Bellman}
    (\mathcal{T}Q)(s,a) = R(s,a) + \gamma \mathbb{E}_{s' \sim \mc{P}(\cdot | s,a)}\left[{\max}_{a' \in \mathcal{A}} Q(s',a')\right].
\end{equation}
We will use the fact that $Q^*$ is the fixed point of $\mc{T}$, i.e., $Q^* = \mc{T} Q^*$, and that $\mc T$ is a $\gamma$-contraction in the $\infty$-norm~\cite{suttonRL}: 
\begin{equation}\label{eqn:Bellmancontraction}
\|\mc T Q_1-\mc T Q_2\|_\infty \le \gamma\|Q_1-Q_2\|_\infty, \forall Q_1,Q_2\in\mathbb{R}^{|\mc S|\times|\mc A|}. 
\end{equation}

Now~\eqref{eq:epoch_relaxed_update} gives
\(Q_{k+1}=(1-\alpha)Q_k+\alpha\,\widehat{\mc T}_kQ_k\),
where, by~\eqref{eq:empirical_operator_def},
\(\medmath{(\widehat{\mc T}_kQ_k)(s,a) = {\texttt{clip}}_{[{-\texttt{G}_k,\texttt{G}_k}]}\!\big(\widehat R_k(s,a)\big)+\gamma\,\widehat\mu_k(s,a)}\).
Since \(Q^\star=\mc TQ^\star\), adding and subtracting \(\mc TQ_k\) then yields
\[
e_{k+1}=(1-\alpha)e_k+\alpha(\mc TQ_k-\mc TQ^\star)+\alpha(\widehat{\mc T}_kQ_k-\mc TQ_k).
\]
Taking $\|\cdot\|_\infty$ on both sides above and using~\eqref{eqn:Bellmancontraction} gives
\begin{equation}
\label{eq:thm_recursion}
\|e_{k+1}\|_\infty
\ \le\
\bigl(1-\alpha(1-\gamma)\bigr)\|e_k\|_\infty
\ +\
\alpha\,\bigl\|\widehat{\mc T}_kQ_k-\mc TQ_k\bigr\|_\infty.
\end{equation}

Now to control the perturbation $\bigl\|\widehat{\mc T}_kQ_k-\mc TQ_k\bigr\|_\infty$, we note from Lemmas~\ref{lem:mu_concentration_all} and~\ref{lem:R_concentration_all} that there exists an event $\mc{G}$ of measure at least $1-\delta$, on which, the following is true for every $(s,a)$ and every epoch $k$:
\begin{equation}
\label{eqn:perturb}
|\widehat \mu_k(s,a) - \mu(s,a)| + |\widehat R_k(s,a) - R(s,a)| \leq \Delta_1 + \Delta_2,    
\end{equation}
where $\Delta_1$ and $\Delta_2$ are as in the statements of Lemmas~\ref{lem:mu_concentration_all} and~\ref{lem:R_concentration_all}, respectively. We will work on this event $\mc{G}$ for the rest of the proof. From the definition of $(\widehat{\mc T}_k Q_k)(s,a)$ in~\eqref{eq:empirical_operator_def} and $(\mc{T}Q_k)(s,a)$ in~\eqref{eqn:Bellman}, we further have for every $(s,a)$: 
\begin{equation}
\label{eqn:perturb_decomp}
\medmath{
\begin{aligned}
|(\widehat{\mc T}_k Q_k)(s,a) - (\mc{T}Q_k)(s,a)| &\leq |\widehat \mu_k(s,a) - \mu_k(s,a)| \\
&+ \bigl|{\texttt{clip}}_{[{-\texttt{G}_k,\texttt{G}_k}]}(\widehat R_k(s,a))-R(s,a)\bigr|.
\end{aligned}}
\end{equation}
Since $|R(s,a)| \leq \tilde{\sigma}$, we have $R(s,a)\in[-\texttt{G}_k,\texttt{G}_k]$ for the \(\texttt{G}_k\) defined in~\eqref{eqn:Gt}, and clipping is precisely the projection onto this interval. As a result, 
\(\medmath{
\bigl|{\texttt{clip}}_{[{-\texttt{G}_k,\texttt{G}_k}]}(\widehat R_k(s,a))-R(s,a)\bigr|
\le |\widehat R_k(s,a)-R(s,a)|.}
\)
Hence, using~\eqref{eqn:perturb} and~\eqref{eqn:perturb_decomp}, on the event $\mc{G}$, we have 
\begin{equation}\label{eq:op_perturb_bound_step1}
{\|\widehat{\mc T}_kQ_k-\mc TQ_k\|_\infty}
 \le \Delta_1 + \Delta_2 := \Delta.\\
\end{equation}
Iterating~\eqref{eq:thm_recursion} and using \(\rho:=1-\alpha(1-\gamma)\) then gives
\(
\|e_K\|_\infty
\le
\rho^K\|e_0\|_\infty
+\frac{\Delta}{1-\gamma}.
\) Since \(1-x\le e^{-x}, \forall x \in (0,1) \) the choice of \(\alpha\) in Theorem~\ref{thm:main theorem 1} implies that \(\rho^K\le (1/T)\). The final form in~\eqref{eq:main-final-complete} then follows by using the expressions for $\Delta_1$ and $\Delta_2$ from Lemmas~~\ref{lem:mu_concentration_all} and \ref{lem:R_concentration_all}, combined with \(T=KH\) and the expression for $K$ in Theorem~\ref{thm:main theorem 1}. 
\end{proof}

\section{Conclusion}
We developed a robust variant of \(Q\)-learning that remains reliable under simultaneous corruption of reward and state feedback. Our finite-time guarantees match those of vanilla \(Q\)-learning up to corruption-dependent bias terms, with a minimax-optimal bias under reward-only corruption. The current implementation requires storing all reward and look-ahead samples within each epoch, leading to auxiliary memory that scales linearly with the epoch length \(H\). A natural direction is to replace these batch estimators with online robust mean estimators that process samples sequentially while maintaining only a compact summary~\cite{shreyas2022robust}. Another important direction is to extend our framework beyond the tabular setting to function approximation.

\appendix
\subsection{Proof of Theorem~\ref{thm:trim_huber_both}}
\label{sec:app}
Here, we provide a proof for the claim in~\eqref{eq:trim_bounded_refined} of Theorem~\ref{thm:trim_huber_both}. Let the data set $\mc{D}'$ be partitioned as $\mc{D}' = \mc{B} \cup \mc{S}$, where $\mc{B}$ and $\mc{S}$ contain the corrupted and clean samples, respectively, within $\mc{D}'$. Then, we have: 
\begin{equation}\medmath{
\hat{\mu}_X^{\textcolor{winered}{\textup{B}}}-\mu_X
=
\frac{1}{M} \underbrace{\sum_{X_i\in\mc B}\bigl(\phi_{\underline{\Gamma},\overline{\Gamma}}(X_i)-\mu_X\bigr)}_{T_1}
+
\frac{1}{M}\underbrace{\sum_{X_i\in\mc S}\bigl(\phi_{\underline{\Gamma},\overline{\Gamma}}(X_i)-\mu_X\bigr)}_{T_2}.}
\label{eqn:mean_decomp}
\end{equation}

We now control $T_1$ and $T_2$ separately by working on an event where the number of corrupted samples does not deviate too much from its mean value of $\varepsilon M.$ To do so, let us define the event \(\mc W:=\{ |\mc{B}| \le \bar\varepsilon M\}\), where $\bar{\varepsilon}$ is the inflated corruption fraction given by:
$ \bar{\varepsilon} := 3\varepsilon/2+16\log(8/\delta)/M. $ Since each sample in $\mc{D}'$ is corrupted independently with probability $\varepsilon$, a simple application of Bernstein's inequality yields $\mathbb{P}(\mc{W}) \geq 1-\delta/4$; see, for instance,~\cite[Appendix~D]{maity2025corruption}. 
Using \(\phi_{\underline{\Gamma},\overline{\Gamma}}(X_i), \mu_X \in[\underline{\Gamma},\overline{\Gamma}]\), on event $\mc{W}$, we have
\begin{equation}
\label{eqn:corrupt_contri}
\frac{|T_1|}{M} \leq \frac{|\mc{B}|}{M} (\overline{\Gamma}-\underline{\Gamma})  \leq \bar\varepsilon(\overline{\Gamma}-\underline{\Gamma}). 
\end{equation}
The above inequality thus holds with probability (w.p.) at least $1-\delta/4.$ Next, to control $T_2$, we note that since $\mc{S}$ only contains clean realizations of the random variable $X$, and $X \in [\underline{\Gamma},\overline{\Gamma}]$, clipping causes no change to the samples in $\mc{S}$, implying:
$T_2 = \sum_{X_i\in\mc S}\bigl(X_i-\mu_X\bigr).$
Since $\{X_i\}$ is a uniformly bounded i.i.d. sequence with mean $\mu_X$, we can control $T_2$ by using Hoeffding's inequality. However, care needs to be taken to account for the fact that $|\mc{S}|$ is a random variable. To proceed, set 
$ t = (\overline{\Gamma}-\underline{\Gamma}) \sqrt{
M \ln\left( \frac{8}{\delta} \right)
}. 
$ Defining $q: = \mathbb{P}(\{|T_2| \geq t\} \cap \mc{W})$, observe that
\begin{equation}
\medmath{
\begin{aligned}
q &= \sum_{j=0}^{M} \mathbb{P}\left(\{|T_2| \geq t\} \cap \mc{W} \cap \{|\mc{S}|=j\}\right)\\
&\overset{(a)}= \sum_{j=0.5 M}^{M} \mathbb{P}\left(\{|T_2| \geq t\} \cap \mc{W} \cap \{|\mc{S}|=j\}\right)\\
& \leq \sum_{j=0.5 M}^{M} \mathbb{P}\left(\{|T_2| \geq t\} | \{|\mc{S}|=j\}\right) \mathbb{P}\left(\{|\mc{S}|=j\}\right)\\
& \overset{(b)}\leq \sum_{j=0.5 M}^{M} 2 \exp(- 2t^2/(j (\overline{\Gamma}-\underline{\Gamma})^2)) \, \mathbb{P}\left(\{|\mc{S}|=j\}\right)\\
& \overset{(c)}\leq 2 \exp(- 2t^2/(M (\overline{\Gamma}-\underline{\Gamma})^2)) \sum_{j=0}^{M}  \mathbb{P}\left(\{|\mc{S}|=j\}\right) \leq \delta/4.\\
\end{aligned}
}
\end{equation}
Here, for (a), we used the fact that on event $\mc{W}$, $|\mc{S}| \geq M - \bar{\varepsilon}M \geq M/2$, since $\bar{\varepsilon} < 0.5$. For (b), we used Hoeffding's inequality, and for (c), we used our choice of $t$. We thus have: $\mathbb{P}\left(|T_2| \geq t\right) \leq q + \mathbb{P}(\mc{W}^{\texttt{C}}) \leq \delta/2.$ Thus, w.p. $1-\delta/2,$
$$ \frac{|T_2|}{M} \leq (\overline{\Gamma}-\underline{\Gamma}) \sqrt{
\frac{1}{M} \ln\left( \frac{8}{\delta} \right)}. $$
 Union-bounding the above event with the one on which~\eqref{eqn:corrupt_contri} holds, using~\eqref{eqn:mean_decomp}, and simplifying using the expression for $\bar \varepsilon$ and $M > \ln(8/\delta)$ leads to the final form in~\eqref{eq:trim_bounded_refined}.
\bibliographystyle{unsrt} 
\bibliography{refs}
\end{document}